\documentclass{article}

\PassOptionsToPackage{numbers, compress}{natbib}

 \usepackage[main, final]{neurips_2026}

\usepackage[utf8]{inputenc} % allow utf-8 input
\usepackage[T1]{fontenc}    % use 8-bit T1 fonts
\usepackage{hyperref}       % hyperlinks
\usepackage{url}            % simple URL typesetting
\usepackage{booktabs}       % professional-quality tables
\usepackage{amsfonts}       % blackboard math symbols
\usepackage{nicefrac}       % compact symbols for 1/2, etc.
\usepackage{microtype}      % microtypography
\usepackage{xcolor}         % colors
\usepackage{amsmath} 
\usepackage{multirow}
\usepackage{graphicx} 
\usepackage[table]{xcolor}
\definecolor{lightblue}{RGB}{220,235,255}
\definecolor{lightyellow}{RGB}{255,250,205}  % softer yellow
\definecolor{lightgreen}{RGB}{220,245,220}   % balanced green
\definecolor{lightred}{RGB}{255,230,230}     % gentle red
\usepackage{etoc}
\addtocontents{toc}{\protect\setcounter{tocdepth}{0}}
\usepackage{pifont}
\usepackage{amsmath,amssymb,amsthm}

\title{McCast: Memory-Guided Latent Drift Correction for Long-Horizon Precipitation Nowcasting}

\author{
  Penghui Wen$^{1}$ \quad Yu Luo$^{1}$ \quad Lintao Wang$^{1}$ \quad Mengwei He$^{1}$ \quad \textbf{Patrick Filippi$^{2}$} \\ \textbf{Thomas Francis Bishop$^{2}$} \quad \textbf{Zhiyong Wang$^{1}$}\thanks{Corresponding author.} \\
  $^{1}$School of Computer Science, The University of Sydney, Australia \\
  $^{2}$School of Life and Environmental Science, The University of Sydney, Australia \\
  \texttt{\{penghui.wen,yluo0465, lintao.wang, mengwei.he\}@sydney.edu.au,} \\
  \texttt{
  \{patrick.filippi, thomas.bishop, zhiyong.wang\}@sydney.edu.au
  }
}

\begin{document}

\maketitle

\begin{abstract}
Existing precipitation nowcasting methods typically adopt an autoregressive formulation, where future states are predicted from previous outputs. However, such an approach accumulates errors over long rollouts, causing forecasts to drift away from physically plausible evolution trajectories. Although various studies have attempted to alleviate this problem by improving step-wise prediction accuracy, they largely neglect the global temporal evolution of meteorological systems and lack mechanisms to actively correct drift during rollouts. To address this issue, we propose McCast, a memory-guided latent drift correction method for precipitation nowcasting. Rather than treating memory as an unordered dictionary of latent states for passive conditioning, McCast leverages temporally organized memory to actively correct autoregressive latent evolution. 
Specifically, McCast introduces a Drift-Corrective Memory Bank (DCBank) that explicitly estimates the temporally consistent drift corrections to calibrate the divergent trajectory. DCBank performs drift correction in two stages: a Corrective Latent Extractor first predicts an initial correction from the current prediction and a reference latent state, and a Correction-Aware Memory Retrieval module then refines the initial correction using temporally organized historical memory. By explicitly correcting latent evolution, instead of improving step-wise prediction accuracy only, McCast produces more temporally coherent and reliable long-horizon forecasts. Experiments on two widely used benchmarks, SEVIR and MeteoNet, show that McCast achieves state-of-the-art performance, particularly in challenging long-horizon forecasting scenarios. \footnote{Code will be made publicly available upon acceptance.}

\end{abstract}

\section{Introduction}
Precipitation nowcasting aims to forecast near-future weather evolution from recent remote-sensing observations, such as weather radar, and plays a critical role in many aspects of society, such as transportation and disaster management~\cite{ravuri2021skilful,zhang2023skilful}.
Recently, deep learning-based approaches~\cite{yu2024diffcast, Lin_2025_CVPR} have advanced precipitation nowcasting with strong performance and efficient inference, commonly framing it as an autoregressive process that predicts future frames step by step from previous outputs~\cite{ravuri2021skilful,gong2024cascast}. However, this formulation ignores the continuous temporal evolution of precipitation systems and treats each prediction step as locally optimal, causing the predicted sequence to progressively deviate from physically plausible evolution trajectories. Such compounding error, referred as \textit{precipitation drift} (Figure~\ref{fig:motivation}) which manifests as degraded spatial structure, motion inconsistency, and loss of coherence in long-horizon forecasts~\cite{yu2025pimmnet}. 

Existing methods mainly focus on mitigating the precipitation drift by improving per-step accuracy through larger models or more sophisticated architectures~\cite{fengperceptually, wen2026duocast}. However, due to the inherent stochasticity of atmospheric system, single-step prediction errors are inevitable, leaving the fundamental error accumulation problem unsolved. It motivates a paradigm shift from relying on locally optimal predictions to explicitly model global temporal evolution. A promising alternative is to introduce external memory~\cite{zhong2024memorybank}, which utilizes historical states as valuable cues about how precipitation intensity and motion evolve over time. In principle, such states can serve as a temporal anchor for long-horizon forecasting. 
However, existing vanilla memory-based designs~\cite{wu2025corgi, yu2025context} generally treat memory as a static and unordered dictionary of latent states for passive conditioning. Although this design allows the model to access historical context, simply using relevant historical content fails to model the temporal dependencies among memory states, and therefore cannot effectively capture the continuity of atmospheric evolution, nor can it provide targeted correction signals for precipitation drift.
Therefore, as illustrated in Figure~\ref{fig:motivation}, passive memory still exhibit noticeable drift, and the corresponding retrieved memories lack clear temporal consistency in both motion and intensity.

\begin{figure}[t]
  \centering
   \includegraphics[width=\linewidth]{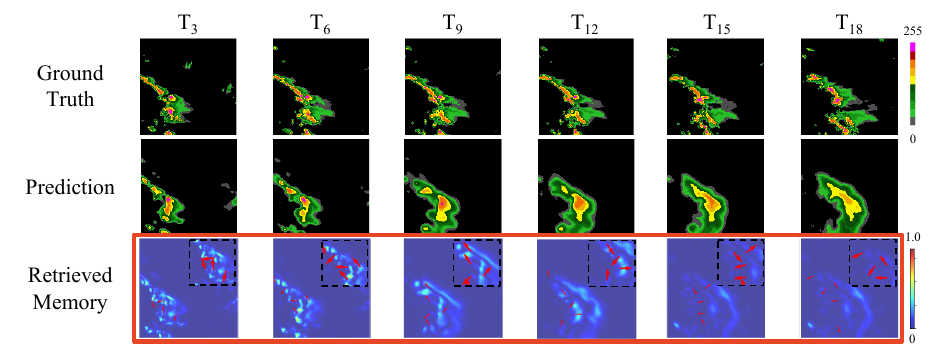}
    \caption{Visualization of DiffCast~\cite{yu2024diffcast} with a vanilla memory bank, including the prediction and retrieved memories. Background color denotes intensity, red arrows indicate motion, and the dotted box highlights a zoomed-in region. The prediction shows drift and the retrieved memory lacks temporal consistency in intensity and motion.}
    
   \label{fig:motivation}
\end{figure}

To address this gap, we propose \textbf{McCast}, a memory-guided latent drift correction method for precipitation nowcasting that explicitly models temporal dependencies among historical memory and reformulates memory from passive conditioning into an \textit{active correction mechanism} for autoregressive latent evolution.
Specifically, McCast introduces a \textbf{Drift-Corrective Memory Bank (DCBank)} that explicitly models the temporal dependencies among historical states and uses them to estimate a drift correction term for the current latent prediction, and updates the prior latent state into a posterior latent state through explicit correction rather than implicit feature fusion. 
DCBank contains two components: a \textbf{Corrective Latent Extractor}, which produces an initial correction by characterizing the discrepancy between the current prediction and a reference latent state, and a \textbf{Correction-Aware Memory Retrieval}, which further refines the initial correction by selectively retrieving temporally organized memory states based on both semantic relevance and drift consistency. In this way, McCast explicitly enforces a temporally consistent evolution and reduces long-horizon precipitation drift.

Our main contributions are summarized as follows:
\begin{itemize}

\item We propose McCast, a memory-guided latent drift correction method that explicitly leverages temporal dependencies in memory to mitigate precipitation drift.

\item We introduce a Drift-Corrective Memory Bank (DCBank) that formulates memory as an explicit correction mechanism rather than a passive conditioning signal, enabling effective calibration of latent representations using historical evolution patterns.

\item We design a two-step drift correction scheme consisting of a Corrective Latent Extractor for initial drift correction estimation and a Correction-Aware Memory Retrieval, to produce refined correction signals for latent evolution.

\item Extensive experiments on two real-world precipitation datasets, SEVIR~\cite{veillette2020sevir} and MeteoNet~\cite{larvor2020meteo}, demonstrate that McCast achieves superior performance and robustness compared to existing methods.
\end{itemize}

\section{Related Works}

\noindent\textbf{Precipitation Nowcasting Models}  typically forecast future frames autoregressively by conditioning on previous outputs. Existing expert models can be broadly divided into deterministic, probabilistic, and hybrid approaches. Deterministic models focus on learning mean spatiotemporal evolution~\cite{shi2015convolutional}, such as AlphaPre~\cite{Lin_2025_CVPR}, which decouples motion and intensity through phase-amplitude factorization, and PercpCast~\cite{fengperceptually}, which improves perceptual quality via rectified flow. Probabilistic models introduce latent variables to capture forecast uncertainty and fine-scale structure~\cite{ravuri2021skilful}, as exemplified by Prediff~\cite{gao2024prediff} and DuoCast~\cite{wen2026duocast}. Hybrid methods, such as CasCast~\cite{gong2024cascast} and DiffCast~\cite{yu2024diffcast}, combine deterministic trajectory prediction with stochastic refinement. Meanwhile, recent weather foundation models have shown strong transferability across forecasting tasks. For example, Aurora~\cite{bodnar2025foundation} has been adapted to diverse downstream applications such as cyclone tracks and weather prediction. 
Despite their effectiveness, these autoregressive approaches suffer from error accumulate over long rollouts, as informative historical states are gradually lost once they fall outside the short prediction window. To address it, we introduce a memory bank mechanism that preserves and retrieves historical states, providing a stable temporal anchor for long-horizon prediction.

\noindent\textbf{Memory Mechanisms in Video Generation} have been widely adopted to preserve long-horizon consistency. \textit{Explicit memory} methods, such as WorldExplorer~\cite{schneider2025worldexplorer} and Memory Forcing~\cite{huang2025memory}, represent video content as 3D point clouds and use reprojection constraints to guide generation. While effective for reducing spatial drift, they are computationally costly and rely on robust 3D reconstruction pipelines~\cite{wang2025vggt}. In contrast, \textit{implicit latent memory} methods store historical information in latent space. For example, Corgi~\cite{wu2025corgi} compresses key frames into latent memory for multi-clip generation and retrieves them via cross-attention, while Context as Memory~\cite{yu2025context} conditions each new frame on viewpoint-overlapping historical frames. LMCast~\cite{gao2025lmcast} leverages pre-trained large language models to retrieve historical information as prior knowledge for prediction.
Despite their effectiveness, existing latent-memory methods typically treat memory as a passive conditioning source and do not explicitly model temporal dependencies among memory states. This is suboptimal for precipitation nowcasting, where future evolution is governed by continuous historical dynamics. We therefore propose a memory-guided latent drift correction method, which exploits sequential dependencies in memory to actively correct accumulated precipitation drift.

\section{Methodology}
\label{section:methodology}

%\subsection{Problem Formulation and Autoregressive Drift} 
\subsection{Overall Architecture}
Precipitation nowcasting predicts future radar frames from historical observations.
Given $\mathbf{X}\in\mathbb{R}^{T_i\times H\times W}$ as input, the model forecasts
$\mathbf{Y}\in\mathbb{R}^{T_o\times H\times W}$ by learning
$p(\mathbf{Y}\mid\mathbf{X})$, where $T_i,T_o$ denote the input and prediction
horizons and $H,W$ denote spatial resolution. By following standard autoregressive
protocols~\cite{yu2024diffcast,fengperceptually}, future frames are generated
rollout by rollout: the latest context window $\mathbf{X}_r$ predicts
$\hat{\mathbf{Y}}_r$, which is then appended to the context. As generated frames increasingly dominate later contexts, errors accumulate and induce latent-space precipitation drift. We maintain a memory bank of corrected historical latent states: $\mathcal{M}_{<r}=\{\mathbf{Z}_1,\ldots,\mathbf{Z}_{r-1}\} \in \mathbb{R}^{R\times L\times D}$, where $R$ is the number of stored historical rollout steps, $L$ is the number of latent tokens, and $D$ is the latent feature dimension. Each memory entry $\mathbf{Z}_i \in\mathbb{R}^{L\times D}$ denotes the corrected posterior latent state from a previous rollout step. In step $r$, the model only retrieves from $\mathcal{M}_{<r}$, ensuring that future information is not used.

\begin{figure}[t]
  \centering
  \includegraphics[width=\linewidth]{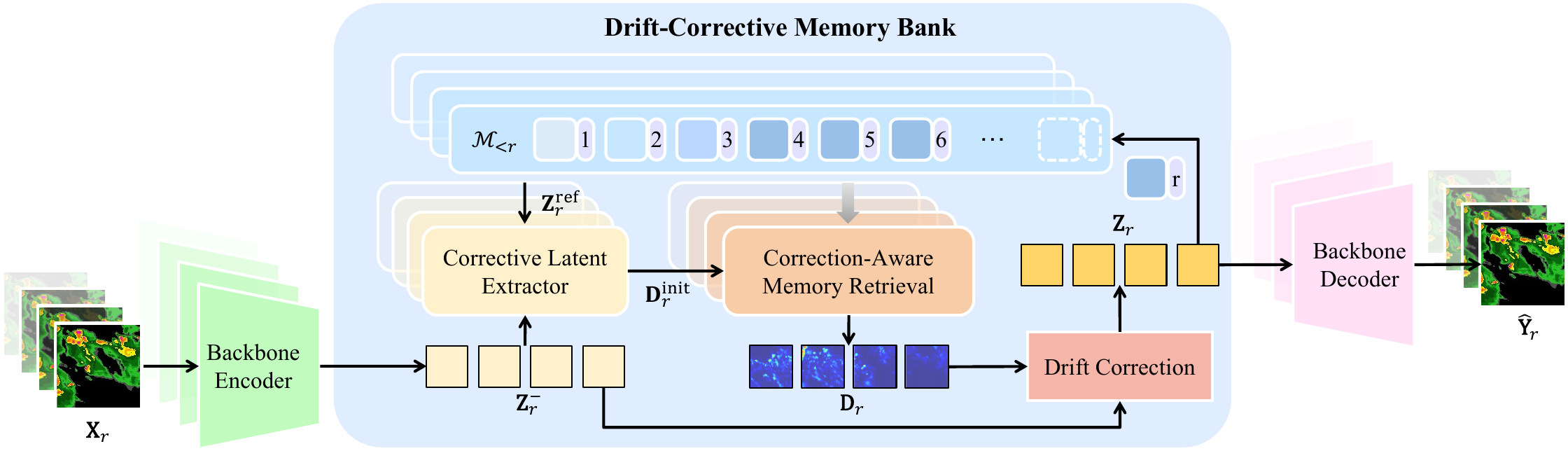}
  \caption{Overview of McCast. A backbone encoder maps the input sequence to a prior latent representation, which is refined by the Drift-Corrective Memory Bank to produce a posterior latent. The posterior latent is then decoded by the backbone decoder to generate the forecast.}
  \label{fig:architecture}
\end{figure}

%\subsection{Overall Architecture}
As shown in Figure~\ref{fig:architecture}, our model follows an encoder-decoder framework consisting of a backbone encoder, a Drift-Corrective Memory Bank module (DCBank), and a backbone decoder. 
Given an input sequence $\mathbf{X}_r$ at rollout step $r$, the backbone encoder first maps it to a prior latent representation $\mathbf{Z}^{-}_r$. 
DCBank then estimates a drift correction term from historical memory and updates the prior latent into a posterior latent representation $\mathbf{Z}_r$, which is decoded to produce the forecast $\hat{\mathbf{Y}}_r$.
Our central design is to use memory as an explicit \emph{correction operator} , rather than as an additional conditioning feature for the decoder. This makes the role of memory precise: it corrects latent-state drift instead of implicitly fusing historical features into prediction. Additional justification is provided in the Appendix.
Formally,
\begin{equation}
\mathbf{Z}^{-}_r = \mathcal{E}(\mathbf{X}_r), \quad
\mathbf{Z}_r = \mathrm{DCBank}(\mathbf{Z}^{-}_r, \mathcal{M}_{<r}), \quad
\mathbf{\hat{Y}}_r = \mathcal{D}(\mathbf{Z}_r),
\end{equation}
where $\mathcal{E}$ and $\mathcal{D}$ denote the backbone encoder and backbone decoder, respectively.
After each rollout step, the posterior latent $\mathbf{Z}_r$ is written back to the memory bank for future retrieval.

\subsection{Drift-Corrective Memory Bank}
The Drift-Corrective Memory Bank (DCBank) estimates an explicit correction term for the current prior latent representation by leveraging historical rollout states (Figure~\ref{fig:architecture2}). Unlike conventional memory modules that use history as passive conditioning, DCBank constrains memory to act through a residual latent correction. From a state-estimation perspective, it transforms the prior latent into a corrected posterior by first estimating a coarse correction and then refining it with temporally organized memory, thereby compensating for accumulated drift.

DCBank consists of two main components: a Corrective Latent Extractor, which estimates an initial drift correction term, and a Correction-Aware Memory Retrieval module, which further refines it using historical memory to obtain a more reliable correction. Formally,
\begin{equation}
\mathbf{D}_r^{\mathrm{init}} = \mathcal{E}_{\delta}(\mathbf{Z}_r^-, \mathbf{Z}_r^{\mathrm{ref}}),  \quad
\mathbf{D}_r = \mathcal{E}_{\psi}(\mathbf{D}_r^{\mathrm{init}}, \mathcal{M}_{<r}), 
\end{equation}
where $\mathcal{E}_{\delta}$ and $\mathcal{E}_{\psi}$ denote the Corrective Latent Extractor and the Correction-Aware Memory Retrieval module, respectively. $\mathbf{Z}_r^{\mathrm{ref}}$ is a temporal reference latent state.

\subsubsection{Corrective Latent Extractor}
The Corrective Latent Extractor estimates an initial drift correction by comparing the current prior latent $\mathbf{Z}_r^-$ with a reference latent $\mathbf{Z}_r^{\mathrm{ref}}$. We use the most recent memory entry as $\mathbf{Z}_r^{\mathrm{ref}}$, and fall back to $\mathbf{Z}_r^-$ when $\mathcal{M}_{<r}$ is empty. To improve robustness, it combines two complementary discrepancy signals: (i) an explicit raw drift residual that directly measures low-level deviation, and (ii) a context-aware discrepancy that captures structured mismatch in representation space. 

Specifically, we define the raw drift residual and the context-aware discrepancy as:
\begin{equation}
\boldsymbol{\delta}_r = \mathbf{Z}_r^- - \mathbf{Z}_r^{\mathrm{ref}} \in \mathbb{R}^{ L \times D},
\qquad
\boldsymbol{\Delta}_r =  \mathbf{W}_{\mathrm{pre}} \cdot \mathbf{Z}_r^- - \mathbf{W}_{\mathrm{ref}} \cdot \mathbf{Z}_r^{\mathrm{ref}}  \in \mathbb{R}^{ L \times D},
\end{equation}
where $\mathbf{W}_{\mathrm{ref}}$ and $\mathbf{W}_{\mathrm{pre}}$ are learnable linear projections. The raw residual $\boldsymbol{\delta}_r$ captures direct deviation from the reference trajectory, whereas $\boldsymbol{\Delta}_r$ captures a more stable, representation-level discrepancy.

To adaptively fuse these two signals, we introduce a gating mechanism conditioned on both the prior latent and the reference latent:
\begin{equation}
\mathbf{G}_\text{init} = \sigma\left(\mathbf{W}_\text{init} \cdot [ \mathbf{Z}^{-}_r , \mathbf{Z}_r^\mathrm{ref}] \right) \in \mathbb{R}^{ L \times D}, 
\end{equation}
\begin{equation}
\mathbf{F}_r = \mathbf{G}_\text{init} \odot (\mathbf{W}_\delta \cdot \boldsymbol{\delta}_r) + (1 - \mathbf{G}_\text{init}) \odot (\boldsymbol{\Delta}_r) \in \mathbb{R}^{ L \times D} ,
\end{equation}
where $[\cdot,\cdot]$ is concatenation, $\mathbf{W}_\delta, \mathbf{W}_\text{init}$ are learnable projections, $\sigma(\cdot)$ is the sigmoid function and $\odot$ denotes element-wise multiplication. 
In this way, the gating mechanism dynamically balances the raw residual and the context-aware discrepancy, yielding a more robust representation of discrepancy. 
Finally, it is mapped to the initial drift correction term $\mathbf{D}^{\text{init}}_r$ via a learnable projection $\mathbf{W}_\text{O}$:
\begin{equation}
\mathbf{D}^{\text{init}}_r = \mathbf{W}_\text{O} \cdot \mathbf{F}_r \in \mathbb{R}^{ L \times D}.
\end{equation}

\begin{figure}[t]
  \centering
  \includegraphics[width=\linewidth]{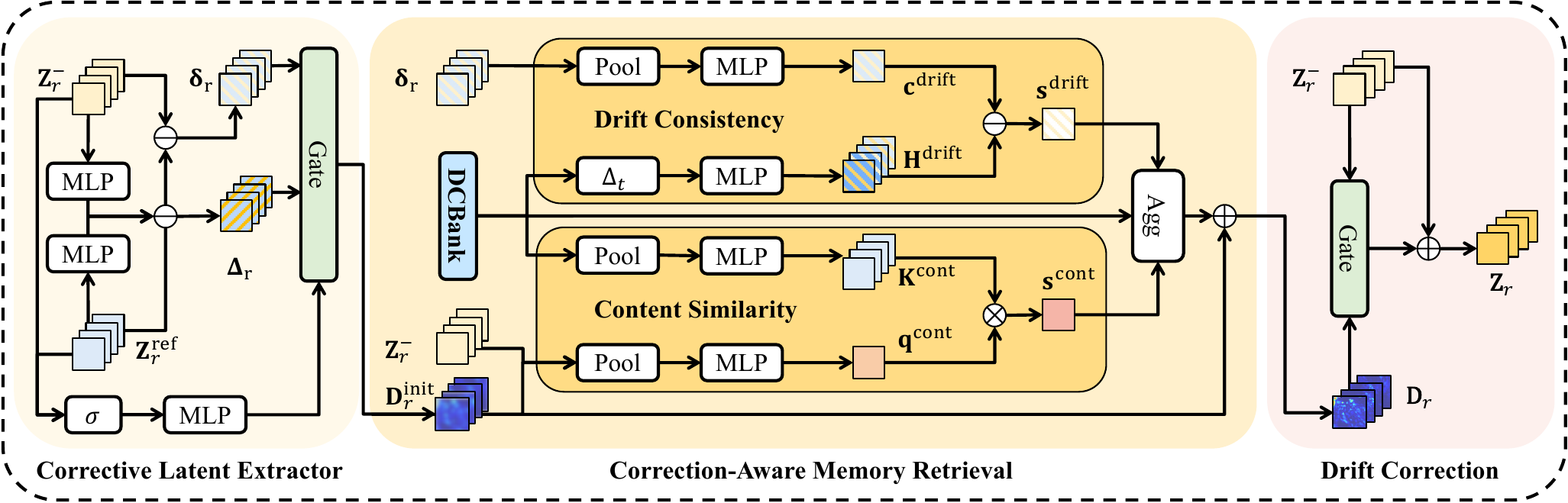}
    \caption{
Details of the Drift-Corrective Memory Bank. An initial drift correction is estimated by the Corrective Latent Extractor, refined through Correction-Aware Memory Retrieval, and applied to the prior latent to obtain the posterior latent.
}
    \label{fig:architecture2}
\end{figure}

\subsubsection{Correction-Aware Memory Retrieval}

The initial correction $\mathbf{D}_r^{\mathrm{init}}$ only uses a local discrepancy between the current prior latent and the most recent reference state. To obtain a more reliable correction, we retrieve information from the historical memory bank. A useful memory state should satisfy two criteria: it should be content-relevant to the current state, and its temporal evolution should be consistent with the current drift pattern. This is important because content-only retrieval may favor the most recent memory entries, which are often closest to the current latent but may also contain stronger accumulated drift.
To preserve temporal order, we add learnable positional embeddings along the rollout dimension to the memory entries and continue to denote the resulting memory bank as $\mathcal{M}_{<r}$ for simplicity.

\noindent\textbf{Content-based retrieval.}
We construct a query from the preliminarily corrected latent state and keys from all memory states:
\begin{equation}
\mathbf{q}^{\text{cont}} = \mathbf{W}_q \cdot \mathrm{Pool}(\mathbf{Z}^{-}_r + \mathbf{D}^{\text{init}}_r) \in \mathbb{R}^{ D}, \quad
\mathbf{K}^{\text{cont}} = \mathbf{W}_K \cdot \mathrm{Pool}(\mathcal{M}_{<r}) \in \mathbb{R}^{ R \times D},
\end{equation}
where $\mathrm{Pool}(\cdot)$ denotes spatial average pooling over $L$ dimension, which summarizes each rollout state into a compact descriptor for stable and efficient memory retrieval.
The resulting content relevance score between preliminarily corrected latent and memories is
\begin{equation}
\mathbf{s}^{\text{cont}} = \frac{\mathbf{q}^{\text{cont}} (\mathbf{K}^{\text{cont}})^{\top}}{\sqrt{D}} \in \mathbb{R}^{ R}.
\end{equation}

\noindent\textbf{Drift-consistency retrieval.}
We characterize the current drift pattern by the raw residual $\boldsymbol{\delta}_r$ and approximate historical drift patterns $\mathbf{H}^{\mathrm{drift}}$ using differences between consecutive memory states. The first entry is padded with zero drift so that the drift sequence matches the memory length. Formally,

\begin{equation}
\mathbf{c}^\text{drift} = \mathbf{W}_c \cdot \mathrm{Pool}(\boldsymbol{\delta}_r) \in \mathbb{R}^{D},
\end{equation}
\begin{equation}
\Delta_t \mathcal{M} = \{\mathbf{0},\mathbf{Z}_2-\mathbf{Z}_1,\ldots,\mathbf{Z}_{r-1}-\mathbf{Z}_{r-2}\}, \quad
\mathbf{H}^{\text{drift}} = \mathbf{W}_h \cdot \mathrm{Pool}(\Delta_t \mathcal{M}) \in \mathbb{R}^{R \times D},
\end{equation}
where $\mathbf{W}_c$ and $\mathbf{W}_h$ denote learnable linear projections. 
The drift consistency score is defined as
\begin{equation}
\mathbf{s}^{\text{drift}} = - \mathrm{Pool}_{D} \left( \left| \mathbf{c}^\text{drift} - \mathbf{H}^{\text{drift}} \right|^2 \right) \in \mathbb{R}^{ R},
\end{equation}
where $\mathrm{Pool}_{D}(\cdot)$ denotes averaging over the feature dimension. This score assigns higher relevance to memory states whose drift patterns are more consistent with the current drift.

\noindent\textbf{Joint relevance and aggregation.}
We combine the two relevance scores and normalize them to obtain the final temporal weights. Here, $\lambda_{\mathrm{drift}}$ controls the relative contribution of the drift-consistency term:
\begin{equation}
\mathbf{s} = \mathrm{softmax}(\mathbf{s}^{\text{cont}} + \lambda_{\text{drift}} \cdot \mathbf{s}^{\text{drift}}) \in \mathbb{R}^{ R}.
\end{equation}
Using these weights, we aggregate the memory along the temporal dimension with a learnable linear projection $\mathbf{W}_\text{agg}$. The resulting representation is then combined with the initial correction term to obtain the memory-enhanced drift correction term:
\begin{equation}
\mathbf{D}_r = \mathbf{W}_{\text{agg}} \left( \mathbf{s} \cdot \mathcal{M}_{<r} \right) + \mathbf{D}^{\text{init}}_r \in \mathbb{R}^{ L \times D}.
\end{equation}

\subsubsection{Drift Correction and Memory Update}
Finally, we integrate the memory-enhanced drift correction term into the prior latent representation via a gating mechanism:
\begin{equation}
\mathbf{G}_\text{corr} = \sigma \left(\mathbf{W}_\text{corr}[\mathbf{Z}^{-}_r, \mathbf{D}_r]\right), \quad
\mathbf{Z}_r = \mathbf{Z}^{-}_r + \mathbf{G}_\text{corr} \odot \mathbf{D}_r.
\end{equation}

The correction gate prevents over-correction by allowing the model to apply memory-derived correction only to latent tokens and channels where drift evidence is reliable.
The resulting $\mathbf{Z}_r$ is then used for subsequent prediction and stored into the memory bank for future rollout steps.

\subsection{Optimization}

Following~\cite{fengperceptually}, we train the model with a joint mean-squared error (MSE) and learned perceptual image patch similarity (LPIPS) objective. The MSE term encourages pixel-wise fidelity between the predicted precipitation sequence and the ground truth, while the LPIPS term promotes perceptual and structural consistency by measuring discrepancies in a learned feature space.  
They are defined as
\begin{equation}
\mathcal{L}_{\mathrm{mse}} = \mathbb{E} \left\| \hat{\mathbf{Y}}-\mathbf{Y} \right\|^2, \quad 
\mathcal{L}_{\mathrm{lpips}} = \mathbb{E} \left[ \sum_l \alpha_l \left\| \phi_l(\hat{\mathbf{Y}})-\phi_l(\mathbf{Y}) \right\|^2 \right].
\end{equation}

Consequently, the overall training objective combines MSE and LPIPS with equal weights:
\begin{equation}
\mathcal{L}
=
\mathcal{L}_{\mathrm{mse}}
+
\mathcal{L}_{\mathrm{lpips}}.
\label{eq:final_loss}
\end{equation}
Here, $\phi_l(\cdot)$ denotes the feature map extracted from the $l$-th layer of a pretrained network for LPIPS calculation, and $\alpha_l$ is the weighting coefficient of the $l$-th layer. More details are in the Appendix.

\section{Experiments \& Discussions}
\label{section:experiments_discussions}

\subsection{Experimental Settings}
\noindent\textbf{Datasets.}
Following common practice~\cite{Lin_2025_CVPR, fengperceptually}, we conduct experiments on two widely used precipitation nowcasting benchmarks, SEVIR~\cite{veillette2020sevir} and MeteoNet~\cite{larvor2020meteo}, both of which provide radar observations of precipitation events. We adopt the standard setting of using 5 historical frames (25 minutes) as input to predict the subsequent 20 frames (100 minutes). The data distributions of the two datasets are summarized in Table~\ref{tab:data_distribution}. More details are provided in the Appendix.

\begin{table}[t]
\centering
\caption{
Dataset statistics. $N_{\mathrm{tr}}$, $N_{\mathrm{val}}$, and $N_{\mathrm{te}}$ denote the training, validation, and test set sizes. Each sample uses 5 input frames to predict 20 future frames over a 100-minute horizon.
}
\label{tab:data_distribution}
\vspace{2mm}
\small
\setlength{\tabcolsep}{5pt}
\begin{tabular}{l|cccccccc}
\toprule
\multirow{2}{*}{Dataset} 
& \multicolumn{3}{c}{Samples} 
& \multirow{2}{*}{Resolution} 
& \multicolumn{2}{c}{Frames} 
& \multicolumn{2}{c}{Spatio-temporal Coverage} \\
\cmidrule(lr){2-4} \cmidrule(lr){6-7} \cmidrule(lr){8-9}
& $N_{\mathrm{tr}}$ 
& $N_{\mathrm{val}}$ 
& $N_{\mathrm{te}}$ 
& $(H,W)$ 
& $T_i$ 
& $T_o$ 
& Spatial 
& Temporal \\
\midrule
SEVIR~\cite{veillette2020sevir}
& 23,808 & 6,016 & 8,100 
& $128{\times}128$ 
& 5 & 20 
& $384{\times}384$ km 
& 5 min \\

MeteoNet~\cite{larvor2020meteo}
& 6,308 & 1,310 & 1,310 
& $128{\times}128$ 
& 5 & 20 
& $550{\times}550$ km 
& 5 min \\
\bottomrule
\end{tabular}%
\end{table}

\noindent\textbf{Evaluation Metrics.}
Following prior works~\cite{Lin_2025_CVPR, fengperceptually}, we evaluate nowcasting performance using CSI, HSS, LPIPS, and SSIM. We report CSI$_\text{M}$, averaged over dataset-specific rainfall thresholds, as well as CSI at the two highest thresholds to assess high-intensity precipitation. HSS measures forecast skill relative to random chance, while SSIM and LPIPS evaluate structural fidelity and perceptual similarity, respectively. Further details are provided in Appendix.

\noindent\textbf{Implementation Details.}
We train all models using AdamW with a learning rate of $5\times10^{-4}$. We set $\lambda_\text{drift}=0.3$ without exhaustive hyperparameter tuning. The backbone encoder $\mathcal{E}$ and decoder $\mathcal{D}$ are initialized from the pretrained Aurora~\cite{bodnar2025foundation} model. During finetuning, we freeze the self-attention layers of the backbone decoder and apply low-rank adaptation (LoRA)~\cite{hu2022lora} with rank 8 to adapt them. All experiments are conducted on three NVIDIA A6000 GPUs and more details are in the Appendix.

\subsection{Overall Performance}

\begin{table*}[t]
\centering
\caption{
Quantitative comparison with state-of-the-art precipitation nowcasting methods on SEVIR and MeteoNet.
Best results are in $\textbf{bold}$, and second-best results are $\underline{underlined}$.
$\dagger$ denotes semi-memory-based modeling, and $\ddagger$ denotes explicit memory-based modeling.
}
\label{tab:sota}
\vspace{1mm}
\small
\setlength{\tabcolsep}{2.2pt}
\renewcommand{\arraystretch}{1.04}
\resizebox{\textwidth}{!}{%
\begin{tabular}{l cccccc cccccc}
\toprule
\multirow{2}{*}{Method} 
& \multicolumn{6}{c}{SEVIR} 
& \multicolumn{6}{c}{MeteoNet} \\
\cmidrule(lr){2-7} \cmidrule(lr){8-13}
& CSI$_\text{M}$$\uparrow$
& CSI$_\text{181}$$\uparrow$
& CSI$_\text{219}$$\uparrow$
& HSS$\uparrow$
& LPIPS$\downarrow$
& SSIM$\uparrow$
& CSI$_\text{M}$$\uparrow$
& CSI$_\text{24}$$\uparrow$
& CSI$_\text{32}$$\uparrow$
& HSS$\uparrow$
& LPIPS$\downarrow$
& SSIM$\uparrow$ \\
\midrule
ConvGRU~\cite{shi2017deep}      & 0.290 & 0.088 & 0.035 & 0.362 & 0.265 & 0.610 & 0.340 & 0.299 & 0.143 & 0.467 & 0.254 & 0.783 \\
MAU~\cite{chang2021mau}        & 0.308 & 0.107 & 0.052 & 0.386 & 0.387 & 0.651 & 0.323 & 0.284 & 0.100 & 0.445 & 0.302 & 0.790 \\
SimVP~\cite{gao2022simvp}    & 0.311 & 0.111 & 0.052 & 0.392 & 0.389 & 0.651 & 0.335 & 0.300 & 0.113 & 0.457 & 0.341 & 0.780 \\
FourCastNet~\cite{pathak2022fourcastnet}  & 0.269 & 0.072 & 0.034 & 0.336 & 0.422 & 0.598 & 0.303 & 0.253 & 0.109 & 0.422 & 0.465 & 0.645 \\
Earthformer~\cite{gao2022earthformer}  & 0.289 & 0.084 & 0.025 & 0.367 & 0.392 & 0.663 & 0.321 & 0.288 & 0.124 & 0.449 & 0.363 & 0.777 \\
PhyDNet~\cite{guen2020disentangling}      & 0.302 & 0.104 & 0.028 & 0.381 & 0.370 & 0.653 & 0.338 & 0.319 & 0.137 & 0.467 & 0.284 & 0.782 \\
EarthFarseer~\cite{wu2024earthfarsser} & 0.300 & 0.099 & 0.041 & 0.383 & 0.383 & 0.633 & 0.340 & 0.317 & 0.137 & 0.473 & 0.300 & 0.754 \\
NowcastNet~\cite{zhang2023skilful}   & 0.279 & 0.077 & 0.035 & 0.351 & 0.404 & \underline{0.684} & 0.343 & 0.321 & 0.160 & 0.475 & 0.288 & 0.788 \\
FACL~\cite{yan2024fourier}     & 0.302 & 0.115 & 0.052 & 0.368 & 0.215 & 0.681 & 0.360 & 0.331 & 0.181 & 0.501 & 0.216 & 0.776 \\
AlphaPre~\cite{Lin_2025_CVPR}     & \underline{0.326} & 0.133 & 0.055 & \underline{0.411} & 0.286 & \textbf{0.688} & \underline{0.382} & 0.363 & \underline{0.200} & \underline{0.516} & 0.199 & \underline{0.797} \\
\midrule
DiffCast~\cite{yu2024diffcast}$^\dagger$
& 0.305 & 0.130 & 0.058 & 0.400 & 0.209 & 0.648 & 0.351 & 0.334 & 0.181 & 0.485 & 0.154 & 0.789 \\
PercpCast~\cite{fengperceptually}$^\dagger$ 
& 0.320 & \underline{0.136} & \underline{0.059} & 0.403 & \underline{0.206} & 0.665 & 0.372 & \underline{0.364} & 0.196 & 0.505 & \underline{0.133} & 0.788 \\
\midrule
\rowcolor{blue!8}
\textbf{McCast (Ours)$^\ddagger$} 
& \textbf{0.339} 
& \textbf{0.172} 
& \textbf{0.107} 
& \textbf{0.438} 
& \textbf{0.196} 
& 0.669 
& \textbf{0.392} 
& \textbf{0.386} 
& \textbf{0.224} 
& \textbf{0.532} 
& \textbf{0.122} 
& \textbf{0.803} \\
\bottomrule
\end{tabular}%
}
\end{table*}

\begin{figure}[t]
  \centering
   \includegraphics[width=\linewidth]{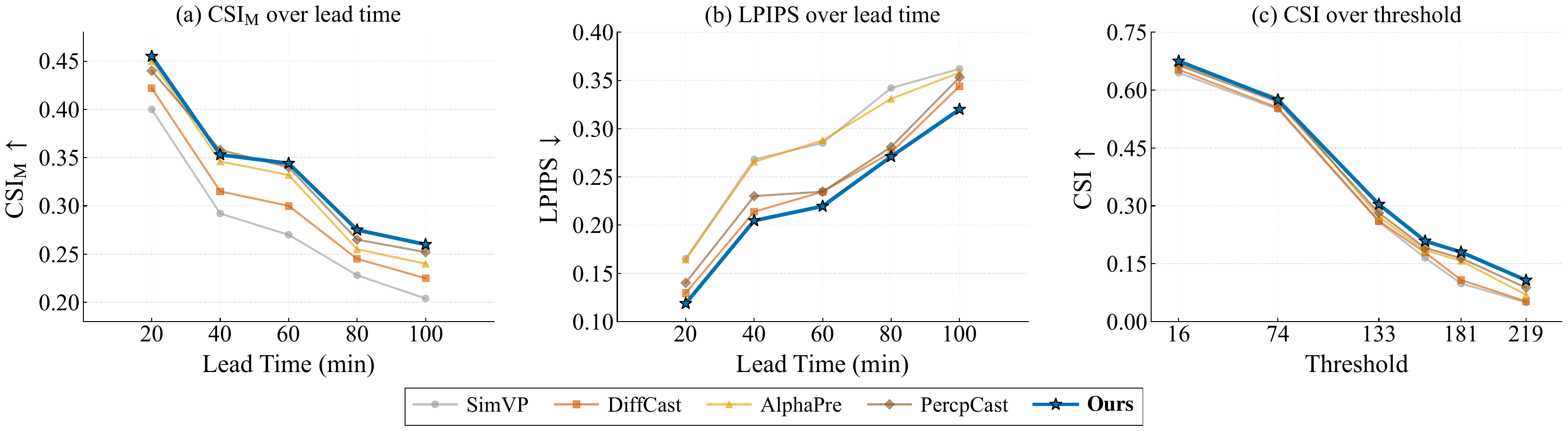}
    \caption{Performance comparison on SEVIR across different lead time and intensity thresholds.}
   \label{fig:time_threshold}
\end{figure}

\noindent\textbf{Quantitative Analysis.}
Table~\ref{tab:sota} compares McCast with recent state-of-the-art precipitation nowcasting methods. We categorize the baselines into methods without explicit memory mechanisms and semi-memory-based methods, where the latter use a global coarse prediction as auxiliary memory during autoregressive prediction. Our method consistently outperforms those methods across both datasets. 
Compared to the strongest baseline, it improves CSI$_\text{M}$ by 5.6\%–5.9\% and HSS by 5.4\%–8.0\%, while remaining competitive on perceptual and structural metrics such as LPIPS and SSIM. Gains become more evident under heavier rainfall conditions: at higher intensity thresholds, our method achieves approximately 5.9\%–21.1\% relative improvements in CSI on SEVIR and MeteoNet, highlighting its stronger capability to capture intense and localized precipitation patterns.
AlphaPre achieves slightly higher SSIM on SEVIR, as SSIM favors smooth, structurally coherent forecasts that align well with its deterministic phase-amplitude design.
Figure~\ref{fig:time_threshold} further corroborates these results. Subplots (a) and (b) show CSI$_\text{M}$ and LPIPS results at different prediction lead times, where our method consistently outperforms recent baselines at nearly all horizons, demonstrating its robustness in autoregressive forecasting settings. Subplot (c) reports the results of CSI under different rainfall intensity thresholds and shows that our method performs particularly well at higher thresholds, demonstrating its effectiveness in more difficult forecasting scenarios.

\noindent\textbf{Qualitative Analysis.}
\begin{figure}[t]
  \centering
   \includegraphics[width=\linewidth]{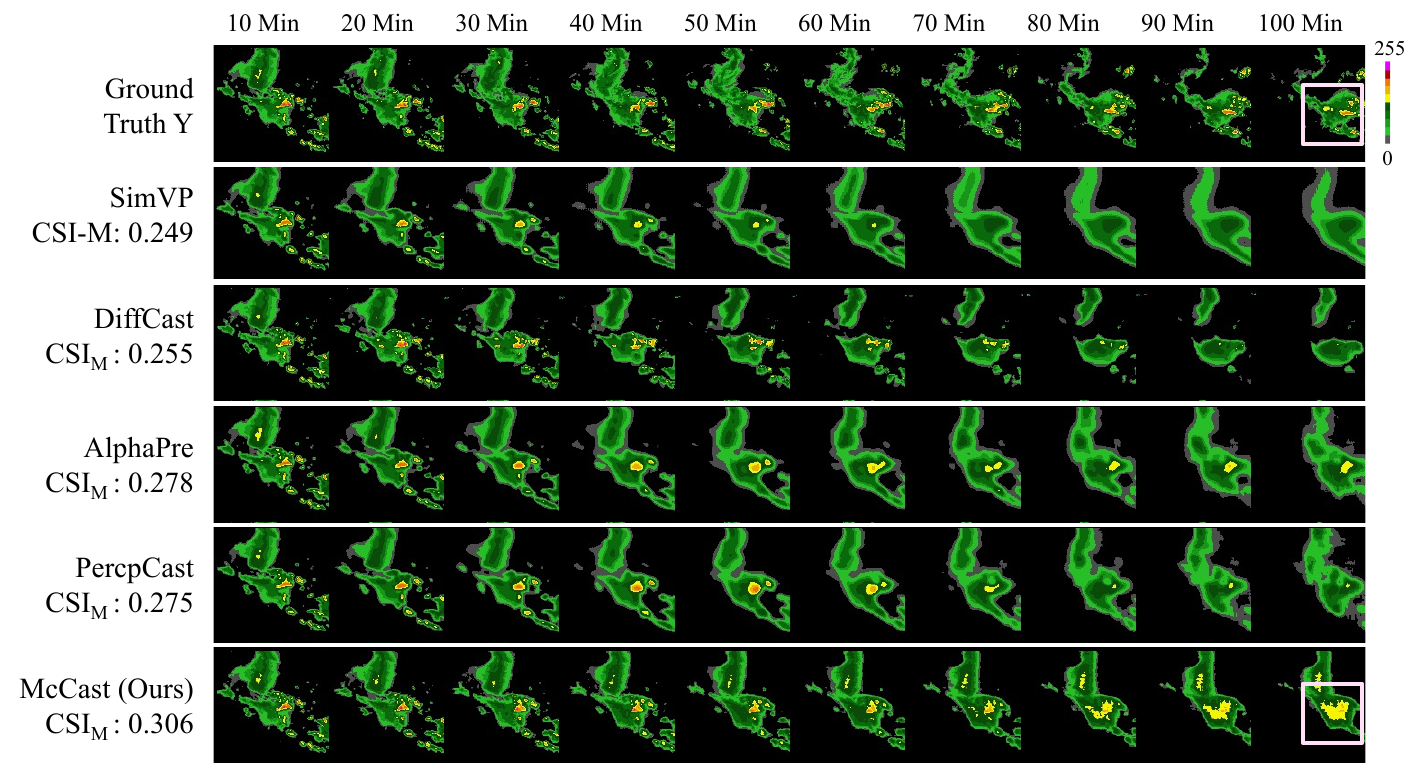}
   \caption{
Qualitative comparison on a SEVIR event. McCast produces forecasts with finer precipitation structures and more coherent temporal evolution than state-of-the-art methods. 
   }
   \label{fig:qualitative}
\end{figure}
Figure~\ref{fig:qualitative} shows qualitative comparisons with state-of-the-art methods on events from SEVIR. Existing methods generally struggle to preserve both coherent large-scale evolution and fine-scale precipitation structure over long forecasting horizons. SimVP~\cite{gao2022simvp} and PhyDNet~\cite{guen2020disentangling} recover the overall precipitation patterns reasonably well at short lead times, but their predictions become progressively blurrier and lose local details as the horizon increases. DiffCast~\cite{yu2024diffcast} and AlphaPre~\cite{Lin_2025_CVPR} better preserve fine-scale structures, yet still exhibit noticeable structural deviations in long-range evolution. In particular, although DiffCast produces sharper local structures, its predicted boundaries tend to drift from the ground-truth evolution. DuoCast~\cite{wen2026duocast} shows a disappearing medium-intensity precipitation region, resulting in an underestimated intensity. By contrast, McCast consistently preserves both the overall precipitation dynamics and local structural details.

\subsection{Ablation Study}

\begin{table*}[t]
\centering
\caption{
Ablation study on SEVIR and MeteoNet. Best results are highlighted in \textbf{bold}. CLE, CAMR, and DC denote the Corrective Latent Extractor, Correction-Aware Memory Retrieval, and Drift Correction module, respectively.
}
\label{tab:ablation}
\vspace{1mm}
\small
\setlength{\tabcolsep}{3.0pt}
\renewcommand{\arraystretch}{1.08}
\resizebox{\textwidth}{!}{%
\begin{tabular}{l cccccc cccccc}

\toprule
\multirow{2}{*}{Variant} & \multicolumn{6}{c}{SEVIR} & \multicolumn{6}{c}{MeteoNet} \\
\cmidrule(lr){2-7} \cmidrule(lr){8-13} & CSI$_\text{M}$$\uparrow$ & CSI$_\text{181}$$\uparrow$

& CSI$_\text{219}$$\uparrow$
& HSS$\uparrow$
& LPIPS$\downarrow$
& SSIM$\uparrow$
& CSI$_\text{M}$$\uparrow$
& CSI$_\text{24}$$\uparrow$
& CSI$_\text{32}$$\uparrow$
& HSS$\uparrow$
& LPIPS$\downarrow$
& SSIM$\uparrow$ \\
\midrule

w/o DCBank      
& 0.314 & 0.143 & 0.072 & 0.415 & 0.285 & 0.668 
& 0.370 & 0.312 & 0.157 & 0.511 & 0.204 & 0.778 \\
w/o CLE        
& 0.317 & 0.157 & 0.076 & 0.420 & 0.219 & 0.658 
& 0.365 & 0.342 & 0.172 & 0.520 & 0.156 & 0.791 \\
w/o CAMR        
& 0.324 & 0.154 & 0.086 & 0.425 & 0.243 & 0.664 
& 0.378 & 0.330 & 0.165 & 0.519 & 0.175 & 0.801 \\
w/o DC 
& 0.321 & 0.149 & 0.082 & 0.428 & 0.250 & 0.651 
& 0.374 & 0.310 & 0.196 & 0.508 & 0.187 & 0.792 \\

\midrule
\rowcolor{blue!8}
\textbf{McCast} 

& \textbf{0.339} & \textbf{0.172} & \textbf{0.107} & \textbf{0.438} & \textbf{0.196} & \textbf{0.669} 

& \textbf{0.392} & \textbf{0.386} & \textbf{0.224} & \textbf{0.532} & \textbf{0.122} & \textbf{0.813} \\

\bottomrule

\end{tabular}%
}
\end{table*}

\noindent\textbf{Effectiveness of the Memory Mechanism.}
\label{ablation_effectiness_of_the_memory_mechanism}
Table~\ref{tab:ablation} (row 1) validates the effectiveness of the proposed memory mechanism for precipitation nowcasting. Removing DCBank consistently degrades performance, with a more pronounced drop under high-intensity precipitation thresholds, indicating its importance for maintaining reliable forecasts in challenging rainfall regimes. Qualitatively, DCBank better preserves fine-grained precipitation structures and intense rainfall cores over long lead times. In contrast, the model without DCBank tends to produce blurred boundaries and attenuated high-intensity responses, suggesting that the memory mechanism helps mitigate autoregressive drift and maintain temporal consistency. More details are provided in the Appendix.

\noindent\textbf{Effectiveness of Corrective Latent Extractor.}
Table~\ref{tab:ablation} (row 2) validates the contribution of the Corrective Latent Extractor (CLE). To isolate its effect, we replace CLE with a vanilla self-attention module for extracting the initial correction term while keeping the remaining components unchanged. The resulting performance drop shows that explicitly modeling the preliminary drift correction is more effective than relying on generic attention-based feature extraction. This suggests that CLE provides a more informative correction prior for subsequent memory retrieval and latent refinement.

\noindent\textbf{Effectiveness of Correction-Aware Memory Retrieval.}
Table~\ref{tab:ablation} (row 3) confirms the importance of Correction-Aware Memory Retrieval (CAMR). Specifically, we remove CAMR and directly apply the initial correction $\mathbf{D}^{\mathrm{init}}_r$ to update the prior latent. Figure~\ref{fig:ablation2} compares both predictions and memory visualizations, where background colors represent intensity changes between successive frames and arrows indicate motion. With CAMR, the retrieved memories exhibit more coherent intensity evolution and motion patterns that are better aligned with the ground truth, enabling the model to preserve precipitation boundaries over time. Without CAMR, the memories show weakened intensity differences and less organized motion, leading to less temporally coherent forecasts. This shows that CAMR effectively leverages sequential dependencies in memory to correct  drift.

\begin{figure}[t]
  \centering
   \includegraphics[width=\linewidth]{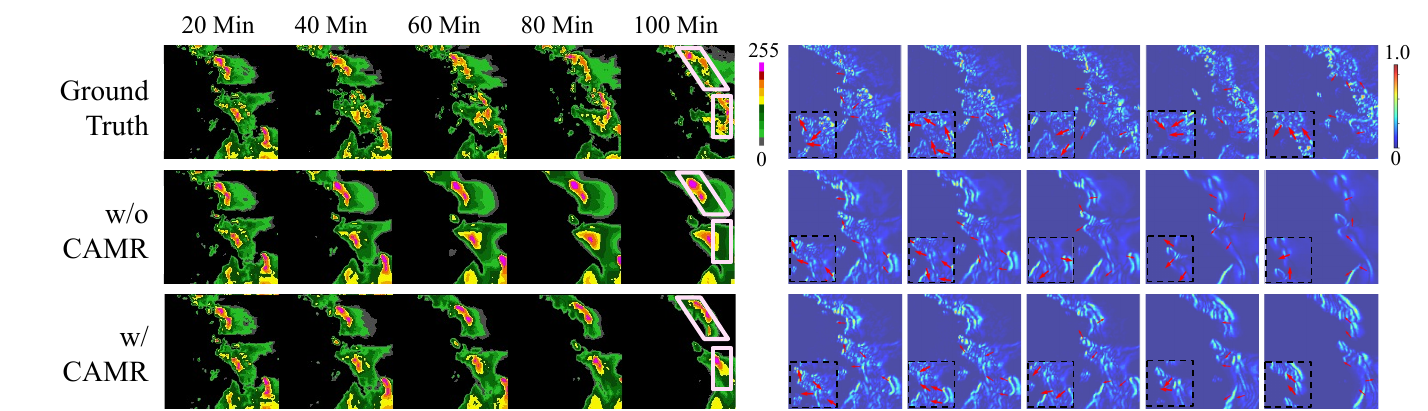}
    \caption{Qualitative comparison on a SEVIR event with and without the correction-aware memory retrieval module. The module better preserves fine-grained precipitation structures and shape. The left panel shows predictions, and the right panel shows the corresponding memory visualizations.}
   \label{fig:ablation2}
\end{figure}

\noindent\textbf{Effectiveness of Active Drift Correction.}
Table~\ref{tab:ablation} (row 4) shows that actively correcting latent evolution with memory outperforms using memory as a passive conditioning signal. For comparison, we remove the drift correction module and instead feed the estimated correction term to the backbone decoder as an additional condition. As shown in Figure~\ref{fig:ablation3},
active drift correction better maintains high-intensity precipitation in the central region at long lead times. The temporal curves also show consistently higher CSI$_\text{M}$ and lower mean absolute error, with the advantage becoming more pronounced from 75 to 100 minutes. This demonstrates that explicitly applying memory-derived corrections to the latent trajectory is crucial for mitigating long-horizon precipitation drift.

\begin{figure}[t]
  \centering
   \includegraphics[width=\linewidth]{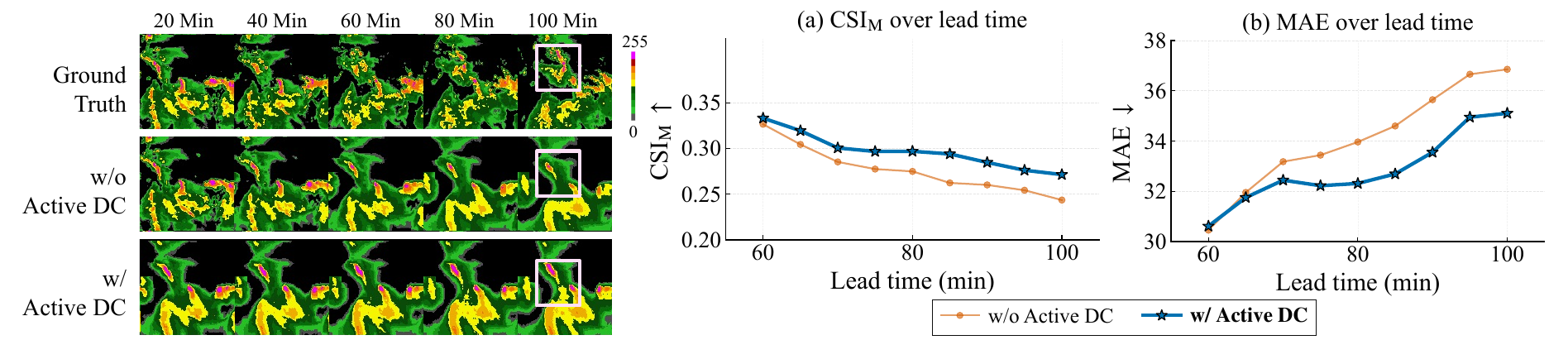}
    \caption{Qualitative comparison on a SEVIR event with and without active drift correction. Active drift correction effectively mitigates precipitation drift, particularly at longer lead times. The left panel shows predicted frames, while the right panel reports CSI$_\text{M}$ and MAE at each prediction timestep.}
   \label{fig:ablation3}
\end{figure}

\section{Conclusion}

We presented McCast, a memory-guided latent drift correction method for precipitation nowcasting. Unlike conventional memory-based methods that use history as passive context, McCast exploits temporally organized latent memory to estimate explicit correction signals through DCBank, thereby mitigating autoregressive drift and improving long-horizon temporal consistency. Experiments on two widely used benchmarks, SEVIR and MeteoNet, demonstrate the effectiveness of McCast, particularly in challenging long-horizon precipitation forecasting scenarios.

\bibliographystyle{unsrtnat}
\bibliography{references}

%%%%%%%%%%%%%%%%%%%%%%%%%%%%%%%%%%%%%%%%%%%%%%%%%%%%%%%%%%%%

\addtocontents{toc}{\protect\setcounter{tocdepth}{2}}
\newtheorem{proposition}{Proposition}
\newtheorem{remark}{Remark}

\newpage
\appendix
\section*{Appendix}

\tableofcontents
\newpage

\section{Broader Impact}
\label{appendix:broader_impact}
This work proposes McCast, a memory-based correction framework for precipitation nowcasting. By explicitly mitigating precipitation drift and improving sequential consistency, McCast can support more reliable short-term forecasts for heavy rainfall, flooding, and other weather-related hazards. More accurate and temporally coherent forecasts may benefit early warning systems, emergency response, transportation planning, and hydrological risk management. The proposed framework is also computationally practical, requiring moderate inference cost and avoiding excessive deployment overhead. This makes it potentially useful for real-time or near-real-time forecasting scenarios on commodity GPUs, which may improve accessibility for operational weather services.

We do not foresee direct negative societal impacts from McCast. However, its use in operational settings should be accompanied by careful regional validation, since forecasting errors or poorly calibrated confidence may affect decisions in safety-critical scenarios. McCast is therefore best viewed as a supportive forecasting component rather than a standalone decision system, and should be continuously evaluated and calibrated under local weather regimes before deployment.

\section{Limitations}
\label{appendix:limitaiton}
Although McCast improves temporal consistency through memory-guided drift correction, it still has several limitations. First, DCBank is designed for short-term autoregressive nowcasting, where recent historical states provide reliable physical anchors; its effectiveness under substantially longer forecasting horizons remains to be further investigated. Second, while McCast explicitly corrects latent drift, the correction is learned from data and does not yet enforce hard physical constraints such as mass conservation or explicit advection dynamics. Nevertheless, our goal is not to replace physics-based constraints, but to provide a lightweight latent-space correction mechanism that can be integrated with either data-driven or physics-informed backbones. Future work will explore more physically constrained memory updates and extend the framework to longer-range precipitation forecasting.

\section{Licenses for Existing Assets}
\label{appendix:licenses_for_existing_assets}

This work builds on several existing assets. We acknowledge their licenses and usage terms as follows:

\begin{itemize}
    \item \textbf{Model License.} Our proposed model will be released under the CC-BY 4.0 license upon publication.
    \item \textbf{Aurora Backbone.} The Aurora backbone used in our experiments is based on the publicly released implementation and pretrained checkpoints from the original work. We cite the corresponding paper in the main text and comply with the license terms specified by the authors.
    \item \textbf{SEVIR Dataset.} We use the publicly available SEVIR dataset, which is accessible through the AWS Open Data Registry: \url{https://registry.opendata.aws/sevir/}.
    \item \textbf{MeteoNet Dataset.} We use the MeteoNet dataset released by Météo-France, the French national meteorological service. The dataset is publicly available at \url{https://meteofrance.github.io/meteonet/}.

\end{itemize}

\section{Justifications}
\subsection{Justification for Active Drift Correction}
\label{appendix:justification_for_active_drift_correction}
Existing memory-based approaches often use historical memory as a passive conditioning signal,
for example by concatenating memory features with the current latent or injecting them through
cross-attention. However, passive conditioning does not explicitly specify how the current prediction should be adjusted. In contrast, active correction treats memory as a source of residual information and uses it to estimate an explicit correction term for the current prior latent features. The following provides a simple justification showing that an explicit residual correction form exposes a transparent sufficient condition for reducing latent drift error. This does not imply that passive conditioning cannot reduce error; rather, passive conditioning does not make the correction direction explicit and therefore does not by itself reveal such a criterion without further architectural or optimization assumptions.

\noindent\textbf{Problem Setup.}
Let $z_t^\star \in \mathbb{R}^d$ denote the target latent state at rollout step $t$, and let $z_t^- \in \mathbb{R}^d$ denote the prior latent predicted from input. The accumulated drift error of the prior is defined as
\begin{equation}
    e_t^- = z_t^- - z_t^\star .
\end{equation}

A passive memory-conditioning model uses historical memory $ \mathcal{M}_{<t}$ as auxiliary context:
\begin{equation}
    \hat{z}_t^{\mathrm{pass}} = g_\theta(z_t^-, \mathcal{M}_{<t}).
\end{equation}
Here, memory is provided to the model, but the model is not explicitly constrained to use memory as a residual correction to the current drift.

By contrast, an active correction model predicts an explicit correction term $\Delta_t$ from the current latent and memory:
\begin{equation}
    \Delta_t = h_\theta(z_t^-, \mathcal{M}_{<t}),
\end{equation}
and updates the prior latent as
\begin{equation}
    \hat{z}_t^{\mathrm{act}} = z_t^- + \Delta_t .
\end{equation}
Ideally, the correction term approximates the negative drift:
\begin{equation}
    \Delta_t \approx z_t^\star - z_t^- = -e_t^- .
\end{equation}

\noindent\textbf{Error Reduction Condition.}
We first show that active correction improves the prior whenever the predicted correction is sufficiently aligned with the negative drift direction.

\begin{proposition}[Sufficient condition for latent drift reduction]
Let $z_t^\star$ be the target latent state and $z_t^-$ be the autoregressive prior latent with drift error $e_t^- = z_t^- - z_t^\star$. Suppose an active correction model produces
\begin{equation}
    \hat{z}_t = z_t^- + \Delta_t .
\end{equation}
If the correction term satisfies: 
\begin{equation}
    \langle e_t^-, \Delta_t \rangle < -\frac{1}{2}\|\Delta_t\|_2^2,
    \label{eq:correction_condition}
\end{equation}
then the corrected latent has strictly smaller squared error than the prior:
\begin{equation}
    \|\hat{z}_t - z_t^\star\|_2^2
    <
    \|z_t^- - z_t^\star\|_2^2 .
\end{equation}
\end{proposition}

\begin{proof}
Since $\hat{z}_t = z_t^- + \Delta_t$, the corrected error is
\begin{equation}
    \hat{z}_t - z_t^\star
    =
    z_t^- + \Delta_t - z_t^\star
    =
    e_t^- + \Delta_t .
\end{equation}
Expanding the squared norm gives
\begin{align}
    \|\hat{z}_t - z_t^\star\|_2^2
    &=
    \|e_t^- + \Delta_t\|_2^2 \\
    &=
    \|e_t^-\|_2^2
    +
    2\langle e_t^-, \Delta_t\rangle
    +
    \|\Delta_t\|_2^2 .
\end{align}
Under condition~\eqref{eq:correction_condition}, we have
\begin{equation}
    2\langle e_t^-, \Delta_t\rangle + \|\Delta_t\|_2^2 < 0.
\end{equation}
Therefore,
\begin{equation}
    \|\hat{z}_t - z_t^\star\|_2^2
    <
    \|e_t^-\|_2^2
    =
    \|z_t^- - z_t^\star\|_2^2 .
\end{equation}
Thus, the active correction strictly reduces the latent drift error.
\end{proof}

\noindent\textbf{Connection to McCast.}
Condition~\eqref{eq:correction_condition} should not be interpreted as an
arbitrary assumption on the learned update. Instead, it characterizes the desired
geometric property of an effective drift correction. In McCast, the autoregressive
prior latent $z_t^-$ may deviate from the target latent trajectory $z_t^\star$,
leading to the drift error $e_t^- = z_t^- - z_t^\star$. Therefore, a useful
correction should move the prior latent in the opposite direction of this drift,
namely toward $z_t^\star - z_t^- = -e_t^-$.

This is consistent with the design of DCBank. Rather than using memory only as
an auxiliary conditioning signal, DCBank explicitly predicts an additive residual
$\Delta_t$ and updates the latent state as
\begin{equation}
    \hat{z}_t = z_t^- + \Delta_t .
\end{equation}
Thus, the role of memory is made directional: it is used to estimate how the
current prior should be adjusted to counteract accumulated drift. The Corrective
Latent Extractor first estimates a coarse correction from the discrepancy between
the current prior latent and a reference latent, providing an initial residual
direction. The Correction-Aware Memory Retrieval module then refines this
residual by retrieving temporally organized historical states according to both
content relevance and drift consistency. As a result, the retrieved memory is not
only visually or semantically relevant to the current latent, but also informative
about the evolution direction needed to restore temporal consistency.

Under this design, Eq.~\eqref{eq:correction_condition} formalizes the property
that McCast aims to learn: the correction should be sufficiently aligned with the
negative drift direction. The term $\langle e_t^-, \Delta_t\rangle$ measures this
alignment. A negative value indicates that $\Delta_t$ points against the drift,
while the additional threshold $-\frac{1}{2}\|\Delta_t\|_2^2$ ensures that the
benefit from moving against the drift is larger than the error introduced by the
magnitude of the update. Therefore, when the learned correction satisfies this
condition, the corrected latent is guaranteed to have smaller squared error than
the autoregressive prior.

Importantly, the condition is not claimed to hold for arbitrary memory updates.
It is a sufficient condition that explains why an explicit residual correction
structure is more principled than passive conditioning. Passive conditioning may
learn to use memory beneficially, but its update direction is implicit and does
not directly expose such an error-reduction criterion. In contrast, McCast makes
the memory-induced update explicit, allowing the correction direction to be
interpreted and theoretically linked to drift reduction.

\noindent\textbf{Passive Conditioning Limitation.}
Passive memory conditioning can be written as
\begin{equation}
    \hat{z}_t^{\mathrm{pass}} = g_\theta(z_t^-, \mathcal{M}_{<t}).
\end{equation}
Although this formulation is expressive, it does not explicitly constrain the model to estimate the residual direction $z_t^\star - z_t^-$. In other words, the memory is used as auxiliary context, but the update direction is implicit.

To guarantee improvement over the prior, one would need
\begin{equation}
    \|\hat{z}_t^{\mathrm{pass}} - z_t^\star\|_2^2
    <
    \|z_t^- - z_t^\star\|_2^2 .
\end{equation}
However, this inequality does not follow from the passive-conditioning form alone. The function $g_\theta$ may learn to use memory beneficially in practice, but without an explicit residual structure or directional constraint, it does not provide a transparent sufficient condition for reducing drift.

By contrast, active correction decomposes the prediction into a prior plus a residual update:
\begin{equation}
    \hat{z}_t^{\mathrm{act}} = z_t^- + \Delta_t .
\end{equation}
This makes the role of memory identifiable: memory is used to estimate a correction direction. As shown above, if this correction is sufficiently aligned with the negative drift direction, error reduction is guaranteed.

\noindent\textbf{Implication for Memory-Based Nowcasting.}
The above analysis supports the design of active memory correction for precipitation nowcasting. Rather than using historical memory only as a passive conditioning signal, McCast uses memory to estimate how the current latent state should be adjusted. This provides two advantages.

First, the correction term gives memory an explicit functional role: compensating for accumulated autoregressive drift. Second, the update form directly links the learned correction to error reduction, since improvement is guaranteed when the correction is aligned with the negative drift direction. Therefore, active correction provides a more principled mechanism for long-horizon temporal consistency than passive memory conditioning.

\subsection{Justification for Imperfect Reference Latents}
The most recent memory entry is used as the reference latent $\mathbf{Z}_r^{\mathrm{ref}}$, but it is not assumed to be error-free. Rather, it provides a temporally adjacent anchor for estimating a coarse local discrepancy with the current prior. To mitigate potential drift in this reference, the correction-aware memory retrieval (CAMR) module refines the initial correction using longer-range historical memory retrieved by both content relevance and drift consistency.

\section{Experimental Details}
\label{appendix:experimental_details}

\subsection{Dataset Details}
\label{appendix:dataset_details}
We provide detailed descriptions of the two widely used precipitation nowcasting benchmarks, SEVIR~\cite{veillette2020sevir} and MeteoNet~\cite{larvor2020meteo}, including their spatial and temporal resolutions, data splits, preprocessing procedures, and evaluation settings. 

\textbf{SEVIR} is an annotated dataset that temporally and spatially aligns multiple meteorological products, including visible satellite imagery, infrared channels, NEXRAD radar mosaics of vertically integrated liquid (VIL), and ground-based lightning detections. Each event corresponds to a 4-hour sequence sampled every 5 minutes over a $384 \text{ km} \times 384 \text{ km}$ region across the continental United States, with radar data at a spatial resolution of 1 km. In this work, we use only the radar (VIL) modality and focus on precipitation events. Following AlphaPre~\cite{Lin_2025_CVPR}, the dataset is split into training, validation, and test sets using January 1, 2019 and June 1, 2019 as temporal cutoffs. Each sequence consists of 25 frames, where 5 observed frames (25 minutes) are used to predict the subsequent 20 frames (100 minutes). All inputs are downsampled to a spatial resolution of $128 \times 128$ to reduce computational cost. For evaluation, frames are rescaled to [0, 255] and binarized using thresholds [16, 74, 133, 160, 181, 219] to compute CSI and HSS.

\textbf{MeteoNet} is a multimodal dataset that includes time series of satellite and radar imagery, numerical weather prediction outputs, and ground observations. It covers a $550 \text{ km} \times 550 \text{ km}$ region over northwestern and southeastern France and spans three years (2016–2018), with radar reflectivity recorded every 5 minutes at a spatial resolution of $0.01^\circ$ (approximately 1.11 km). In this work, we use only the radar reflectivity data. Following AlphaPre~\cite{Lin_2025_CVPR}, the data are split into training, validation, and test sets using January 1, 2018 and June 1, 2018 as temporal cutoffs. All sequences are downsampled to $128 \times 128$ resolution for computational efficiency. For evaluation, we adopt thresholds [12, 18, 24, 32] for CSI and HSS, consistent with AlphaPre~\cite{Lin_2025_CVPR}.

\subsection{Baseline Details}
\label{appendix:baseline_details}

The baseline results reported in Table~\ref{tab:sota} are mainly taken from AlphaPre~\cite{Lin_2025_CVPR}, which provides a comprehensive comparison with representative precipitation nowcasting methods under the same evaluation protocol. For FACL~\cite{yan2024fourier} and PercpCast~\cite{fengperceptually}, we reproduce the results using their released papers and code implementations. All reproduced baselines are evaluated following the same data splits and metrics used for McCast whenever applicable.

\subsection{Backbone Details}
\label{appendix:backbone_details}

Inspired by the strong generalization capability of foundation models across diverse domains~\cite{shi2026tracerouter,dou2026dna}, we adopt a foundation-model backbone for precipitation nowcasting.
Aurora~\cite{bodnar2025foundation} is a large-scale 3D foundation model for Earth-system forecasting, pretrained on over one million hours of heterogeneous meteorological data. Architecturally, it follows an encoder-backbone-decoder design, consisting of a 3D Perceiver encoder, a multi-scale 3D Swin Transformer U-Net backbone~\cite{liu2022swin}, and a 3D Perceiver decoder. We adopt the Aurora $0.25^\circ$ Pretrained checkpoint, which contains approximately 1.3B parameters and operates on a global grid with $0.25^\circ$ spatial resolution, corresponding to an input size of $721{\times}1440$, and a 6-hour temporal interval. It ingests multiple meteorological variables, including surface-level variables, static variables, and atmospheric variables across multiple pressure levels, and generates the corresponding future predictions. This checkpoint is pretrained on multiple large-scale meteorological datasets, including ERA5~\cite{hersbach2023era5,hersbach2020era5}, HRES T0~\cite{rasp2024weatherbench}, and other heterogeneous weather data sources. In its default forecasting setup, Aurora uses two historical frames to predict the next future frame for each rollout.

Specifically, $\mathcal{E}$ is instantiated with Aurora's 3D Perceiver encoder, while $\mathcal{D}$ consists of the multi-scale 3D Swin Transformer U-Net followed by the 3D Perceiver decoder. Table~\ref{tab:aurora_mccast_architecture} provides a module-wise comparison between the original Aurora architecture and McCast. We initialize the model from the Aurora $0.25^\circ$ Pretrained checkpoint. To adapt Aurora to precipitation nowcasting, we treat the precipitation map as a new surface-level variable. Since our task uses only precipitation fields as input and predicts only future precipitation, the original input variables and output targets of Aurora are not directly applicable. Therefore, we train a task-specific input projector in the 3D Perceiver encoder and a task-specific output head in the 3D Perceiver decoder, following the original projector and head designs, while retaining the pretrained Transformer-based backbone for spatiotemporal modeling. Although our input resolution of $128{\times}128$ differs from Aurora's original global grid resolution of $721{\times}1440$, the Transformer-based architecture can naturally handle variable spatial token lengths after patch embedding, enabling adaptation to the precipitation nowcasting resolution. We follow Aurora’s default forecasting setup, using two historical frames to predict the next future frame for each rollout.

Additionally, the performance gains are not specific to the Aurora backbone. To verify the generality of DCBank, we conduct additional experiments with alternative backbones, with results reported in Appendix~\ref{effectiveness_of_dcbank_on_different_backbone}.

\begin{table}[!htbp]
\centering
\caption{
Module-wise summary of Aurora and McCast. $D$ and $H$ denote the number of blocks and attention heads, respectively. The highlighted row indicates the additional drift-corrective memory bank (DCBank) module introduced by McCast.
}
\label{tab:aurora_mccast_architecture}
\vspace{1mm}
\small
\setlength{\tabcolsep}{5.5pt}
\renewcommand{\arraystretch}{1.10}
\begin{tabular}{lllcc}
\toprule
Model & Component & Module & $D$ & $H$ \\
\midrule
\multirow{9}{*}{Aurora}
& Embedding 

& Patch embedding, $4{\times}4$, $d=512$ 
& -- & -- \\

\cmidrule(lr){2-5}

& Encoder 
& 3D Perceiver encoder
& 1 & -- \\

\cmidrule(lr){2-5}
& \multirow{6}{*}{Backbone}
& Swin encoder (block 1) 
& 6 & 8 \\

& 
& Swin encoder (block 2) 
& 10 & 16 \\

& 
& Swin encoder (block 3)
& 8 & 32 \\

& 
& Swin decoder (block 1) 
& 8 & 32 \\

& 
& Swin decoder (block 2) 
& 10 & 16 \\

& 
& Swin decoder (block 3) 
& 6 & 8 \\

\cmidrule(lr){2-5}
& Decoder 
& 3D Perceiver decoder 
& 1 & 16 \\

\midrule

\multirow{10}{*}{\textbf{McCast}}
& Embedding 
& Patch embedding, $4{\times}4$, $d=512$ 
& -- & -- \\

\cmidrule(lr){2-5}

& Encoder 
& 3D Perceiver encoder
& 1 & -- \\

\cmidrule(lr){2-5}
& \cellcolor{blue!8}\textbf{DCBank}
& \cellcolor{blue!8}\textbf{DriftAwareLightMemory}
& \cellcolor{blue!8}\textbf{--}
& \cellcolor{blue!8}\textbf{--} \\

\cmidrule(lr){2-5}
& \multirow{6}{*}{Backbone}
& Swin encoder (block 1) 
& 6 & 8 \\

& 
& Swin encoder (block 2) 
& 10 & 16 \\

& 
& Swin encoder (block 3) 
& 8 & 32 \\

& 
& Swin decoder (block 1) 
& 8 & 32 \\

& 
& Swin decoder (block 2) 
& 10 & 16 \\

& 
& Swin decoder (block 3) 
& 6 & 8 \\

\cmidrule(lr){2-5}
& Decoder 
& 3D Perceiver decoder
& 1 & 16 \\
\bottomrule
\end{tabular}
\end{table}

\subsection{Compute Resources}
\label{appendix:compute_resources}

All training experiments are conducted on three NVIDIA A6000 GPUs, each with 48GB of memory, using a batch size of 16. The training machine is equipped with an Intel(R) Xeon(R) Gold 6336Y CPU @ 2.40GHz. Training McCast requires approximately 30GB of GPU memory. All evaluation experiments are performed on a single NVIDIA RTX 4090 GPU with 24GB of memory. The evaluation machine is equipped with a 13th Gen Intel(R) Core(TM) i9-13900 CPU.

\subsection{Evaluation Metrics}
Following prior work~\cite{Lin_2025_CVPR}, we evaluate forecasting performance using the Critical Success Index (CSI), Heidke Skill Score (HSS), Learned Perceptual Image Patch Similarity (LPIPS), and Structural Similarity Index (SSIM). The formal definitions are provided below.

\noindent\textbf{Critical Success Index (CSI)} measures the overlap between predicted and ground-truth precipitation events while ignoring correct negatives. Let $\hat{\mathbf{Y}} \in \mathbb{R}^{T_o \times H \times W}$ and $\mathbf{Y} \in \mathbb{R}^{T_o \times H \times W}$ denote the predicted and ground-truth precipitation sequences, respectively. Given a precipitation threshold $\tau$, we first binarize each pixel as
\begin{equation}
\hat{\mathbf{B}}_{\tau} = \mathbb{I}(\hat{\mathbf{Y}} \ge \tau), 
\qquad
\mathbf{B}_{\tau} = \mathbb{I}(\mathbf{Y} \ge \tau),
\end{equation}
where $\mathbb{I}(\cdot)$ is the indicator function. The contingency statistics are then computed as
\begin{equation}
\mathrm{TP}_{\tau}
=
\sum_{t,h,w}
\mathbb{I}\!\left(\hat{\mathbf{B}}_{\tau}^{t,h,w}=1 \land \mathbf{B}_{\tau}^{t,h,w}=1\right),
\end{equation}
\begin{equation}
\mathrm{FN}_{\tau}
=
\sum_{t,h,w}
\mathbb{I}\!\left(\hat{\mathbf{B}}_{\tau}^{t,h,w}=0 \land \mathbf{B}_{\tau}^{t,h,w}=1\right),
\end{equation}
\begin{equation}
\mathrm{FP}_{\tau}
=
\sum_{t,h,w}
\mathbb{I}\!\left(\hat{\mathbf{B}}_{\tau}^{t,h,w}=1 \land \mathbf{B}_{\tau}^{t,h,w}=0\right).
\end{equation}
Here, $\mathrm{TP}_{\tau}$ denotes correctly predicted precipitation pixels, $\mathrm{FN}_{\tau}$ denotes missed precipitation pixels, and $\mathrm{FP}_{\tau}$ denotes falsely predicted precipitation pixels.

For a threshold $\tau$, CSI is defined as
\begin{equation}
\mathrm{CSI}_{\tau}
=
\frac{\mathrm{TP}_{\tau}}
{\mathrm{TP}_{\tau} + \mathrm{FN}_{\tau} + \mathrm{FP}_{\tau}}.
\end{equation}
A higher CSI indicates better precipitation event detection. Following common practice, we report CSI at multiple thresholds and use CSI$_{\mathrm{M}}$ to denote the mean CSI over all selected thresholds:
\begin{equation}
\mathrm{CSI}_{\mathrm{M}}
=
\frac{1}{|\mathcal{T}|}
\sum_{\tau \in \mathcal{T}}
\mathrm{CSI}_{\tau},
\end{equation}
where $\mathcal{T}$ is the set of evaluation thresholds.

\noindent\textbf{Heidke Skill Score (HSS)} evaluates the categorical forecasting skill while accounting for random chance. For a given threshold $\tau$, it is computed as

\begin{equation}
\mathrm{HSS}_{\tau}
=
\frac{
2(\mathrm{TP}_{\tau}\mathrm{TN}_{\tau} - \mathrm{FN}_{\tau}\mathrm{FP}_{\tau})
}{
(\mathrm{TP}_{\tau}+\mathrm{FN}_{\tau})(\mathrm{FN}_{\tau}+\mathrm{TN}_{\tau})
+
(\mathrm{TP}_{\tau}+\mathrm{FP}_{\tau})(\mathrm{FP}_{\tau}+\mathrm{TN}_{\tau})
}.
\end{equation}

We compute HSS over the same threshold set $\mathcal{T}$ used for CSI and report the averaged score:

\begin{equation}
\mathrm{HSS}
=
\frac{1}{|\mathcal{T}|}
\sum_{\tau \in \mathcal{T}}
\mathrm{HSS}_{\tau}.
\end{equation}

Higher HSS indicates stronger overall skill in distinguishing precipitation and non-precipitation regions across different rainfall intensities.

\noindent\textbf{Learned Perceptual Image Patch Similarity (LPIPS)} measures perceptual discrepancy between the predicted sequence $\hat{\mathbf{Y}}$ and the ground truth $\mathbf{Y}$ in a deep feature space. Given a pretrained feature extractor $\phi_l(\cdot)$ at layer $l$, LPIPS is generally written as

\begin{equation}
\mathrm{LPIPS}(\hat{\mathbf{Y}}, \mathbf{Y})
=
\sum_l
\frac{1}{H_l W_l}
\sum_{h,w}
\alpha_l \left\|
\left(
\phi_l(\hat{\mathbf{Y}})_{h,w}
-
\phi_l(\mathbf{Y})_{h,w}
\right)
\right\|_2^2,
\end{equation}

where $\alpha_l$ denotes the weighting coefficient associated with layer $l$. Lower LPIPS indicates better perceptual similarity and sharper reconstructed precipitation structures.

\noindent\textbf{Structural Similarity Index (SSIM)} measures structural similarity between the prediction and ground truth by comparing local luminance, contrast, and structure. For a predicted frame $\hat{\mathbf{Y}}_t$ and the corresponding ground-truth frame $\mathbf{Y}_t$, SSIM is defined as

\begin{equation}
\mathrm{SSIM}_t(\hat{\mathbf{Y}}_t, \mathbf{Y}_t)
=
\frac{
(2\mu_{\hat{\mathbf{Y}}_t}\mu_{\mathbf{Y}_t}+c_1)
(2\sigma_{\hat{\mathbf{Y}}_t\mathbf{Y}_t}+c_2)
}{
(\mu_{\hat{\mathbf{Y}}_t}^2+\mu_{\mathbf{Y}_t}^2+c_1)
(\sigma_{\hat{\mathbf{Y}}_t}^2+\sigma_{\mathbf{Y}_t}^2+c_2)
},
\end{equation}

where $\mu$, $\sigma^2$, and $\sigma_{\hat{\mathbf{Y}}_t\mathbf{Y}_t}$ denote the local mean, variance, and covariance, respectively, and $c_1,c_2$ are constants for numerical stability. We compute SSIM independently for each predicted frame and then average the frame-wise scores over the full forecasting horizon:

\begin{equation}
\mathrm{SSIM}
=
\frac{1}{T_o}
\sum_{t=1}^{T_o}
\mathrm{SSIM}_t(\hat{\mathbf{Y}}_t, \mathbf{Y}_t).
\end{equation}

This sequence-level average reflects the overall structural fidelity of the predicted precipitation evolution across all lead times.

\subsection{Training and Inference}
\label{appendix:training_inference}

During training, we unroll the autoregressive prediction process over the target horizon. For each training sequence, the memory bank is initialized as empty. At rollout step $r$, the model encodes the current context window, applies DCBank if historical memory is available, decodes the corrected latent state, and writes the posterior latent into memory. The next context window is then constructed by sliding the window forward and appending the newest prediction. The forecasting loss is accumulated over all rollout steps. In our implementation, the drift correction module is trained end-to-end through the rollout forecasting objective. Since memory information can only affect the decoder through the gated residual correction, the model is encouraged to learn corrections that improve long-horizon prediction rather than arbitrary memory fusion.

Inference follows the same autoregressive procedure as training. The memory bank is also initialized empty and updated online using only previously generated posterior latent states. Since DCBank retrieves only from $\mathcal{M}_{<r}$ at step $r$, the inference process does not use future frames or ground-truth future latents.

\section{Model Complexity}
\label{appendix:model_complexity}

Table~\ref{tab:complexities} compares McCast with recent baselines in terms of parameter count and inference-time complexity on SEVIR for 100-minute forecasting, evaluated on a single NVIDIA RTX 4090 GPU. McCast has the smallest parameter count among the compared models and requires only 3.80 TFLOPs, making it computationally efficient and practical for modern commodity GPUs.

Compared with DiffCast, McCast uses fewer parameters and substantially lower computational cost, indicating better inference efficiency. Relative to AlphaPre, McCast has a smaller parameter scale with only a slightly higher FLOPs cost. This cost remains well within the capability of modern hardware: for example, an RTX 4090 provides up to 82.6 TFLOPs in FP32 and 330.4 TFLOPs in FP16, while McCast requires only 3.80 TFLOPs. 

In addition to FLOPs, McCast also demonstrates competitive runtime efficiency. On the same hardware, it takes 0.20 seconds per sample to generate 20 forecast frames, corresponding to the 100-minute SEVIR horizon, whereas DiffCast requires 3.67 seconds per sample. These results show that McCast achieves a favorable balance between forecasting performance, model capacity, and practical efficiency for real-time or near-real-time precipitation nowcasting.

\begin{table}[!htbp]
\centering
\caption{Model complexity comparison with state-of-the-art precipitation nowcasting methods. 
The number of parameters and computational cost are reported in millions and TFLOPs, respectively.}
\label{tab:complexities}
\vspace{1mm}
\small
\setlength{\tabcolsep}{8pt}
\renewcommand{\arraystretch}{1.08}
\begin{tabular}{lccc}
\toprule
Method & Params (M) $\downarrow$ & TFLOPs $\downarrow$ & Year \\
\midrule
ConvGRU~\cite{shi2017deep}      & 18.21  & 0.03  & 2017 \\
PhyDNet~\cite{guen2020disentangling}      & 11.80  & 0.08  & 2020 \\
MAU~\cite{chang2021mau}          & 20.13  & 0.09  & 2021 \\
STRPM~\cite{chang2022strpm}        & 439.63 & 0.28  & 2022 \\
SimVP~\cite{gao2022simvp}        & 44.25  & 0.05  & 2022 \\
Earthformer~\cite{gao2022earthformer}  & 34.61  & 0.04  & 2022 \\
PreDiff~\cite{gao2024prediff}      & 135.24 & 2.80  & 2023 \\
DiffCast~\cite{yu2024diffcast}     & 58.33  & 72.49 & 2024 \\
FACL~\cite{yan2024fourier}         & 14.38  & 0.02  & 2024 \\
AlphaPre~\cite{Lin_2025_CVPR}     & 89.03  & 1.56  & 2025 \\
PercpCast~\cite{fengperceptually}    & 55.87  & 1.23  & 2025 \\
\midrule
\rowcolor{blue!8}
\textbf{McCast} (LoRA r=8) & \textbf{10.24} & 3.80 & 2026 \\
\bottomrule
\end{tabular}
\end{table}

\section{Additional Ablation Study}
\label{appendix:additional_ablation_study}

\subsection{Advantages of Weather Foundation Models}
We use Aurora~\cite{bodnar2025foundation} as the backbone of our framework, motivated by the strong transferability of recent weather foundation models. As shown in Table~\ref{tab:ablation_aurora}, Aurora, when adapted to precipitation nowcasting using LoRA, achieves competitive performance compared with the state-of-the-art, such as AlphaPre~\cite{Lin_2025_CVPR}. In particular, it improves CSI under high-intensity precipitation thresholds, as well as HSS and LPIPS, while maintaining comparable CSI$_{\mathrm{M}}$. Moreover, LoRA finetuning enables Aurora to achieve strong performance with only about $8\%$ of the trainable parameters used by AlphaPre. These results suggest that Aurora provides a strong representation prior for precipitation evolution, making it a suitable backbone for our proposed memory-based drift correction framework.

\begin{table}[!htbp]
\centering
\caption{
Quantitative Performance on SEVIR. Best results are highlighted in \textbf{bold}.
}
\label{tab:ablation_aurora}
\vspace{1mm}
\small
\setlength{\tabcolsep}{3.5pt}
\renewcommand{\arraystretch}{1.08}
\begin{tabular}{llccccccc}
\toprule
\multirow{2}{*}{Model} 
& \multirow{2}{*}{Finetune Method}
& \multirow{2}{*}{Train Params (M)}
& \multicolumn{6}{c}{SEVIR} \\
\cmidrule(lr){4-9}
& 
& 
& CSI$_\text{M}$$\uparrow$ 
& CSI$_\text{181}$$\uparrow$
& CSI$_\text{219}$$\uparrow$
& HSS$\uparrow$
& LPIPS$\downarrow$
& SSIM$\uparrow$ \\
\midrule

PercpCast  
& Full finetune 
& 55.87 
& 0.320 & 0.136 & 0.059 & 0.403 & 0.206 & 0.665 \\

AlphaPre  
& Full finetune 
& 89.03 
& 0.326 & 0.133 & 0.055 & 0.411 & 0.286 & \underline{0.688} \\

\midrule
\multirow{3}{*}{Aurora}
& LoRA (r=8)
& 6.56
& 0.314 & 0.143 & 0.072 & 0.415 & 0.285 & 0.668 \\

& Full finetune
& 1243.13
& 0.330 & 0.162 & 0.086 & 0.429 & 0.249 & 0.674 \\

& Full finetune (scratch)
& 1243.13
& 0.324 & 0.156 & 0.077 & 0.420 & 0.253 & 0.670 \\

\midrule
\rowcolor{blue!8}
\textbf{McCast} 
& LoRA (r=8)
& \textbf{10.24}
& \underline{0.339} 
& \underline{0.172} 
& \underline{0.107} 
& \underline{0.438} 
& \underline{0.196} 
& 0.669 \\

\rowcolor{blue!8}
\textbf{McCast} 
& Full finetune
& 1243.13
& \textbf{0.349} 
& \textbf{0.193} 
& \textbf{0.112} 
& \textbf{0.449} 
& \textbf{0.177} 
& \textbf{0.700}\\
\bottomrule
\end{tabular}
\end{table}

\subsection{Effectiveness of the Memory Mechanism}
Complementing the quantitative ablation in Section~\ref{ablation_effectiness_of_the_memory_mechanism}, Figure~\ref{fig:ablation1} provides a qualitative comparison of McCast with and without DCBank. The model equipped with DCBank better preserves fine-grained precipitation structures and high-intensity rainfall cores over long lead times, particularly in the highlighted regions. In contrast, removing DCBank leads to blurred boundaries and weakened intense responses, indicating that the memory mechanism helps mitigate autoregressive drift and maintain temporal consistency.

\begin{figure}[!htbp]
  \centering
   \includegraphics[width=\linewidth]{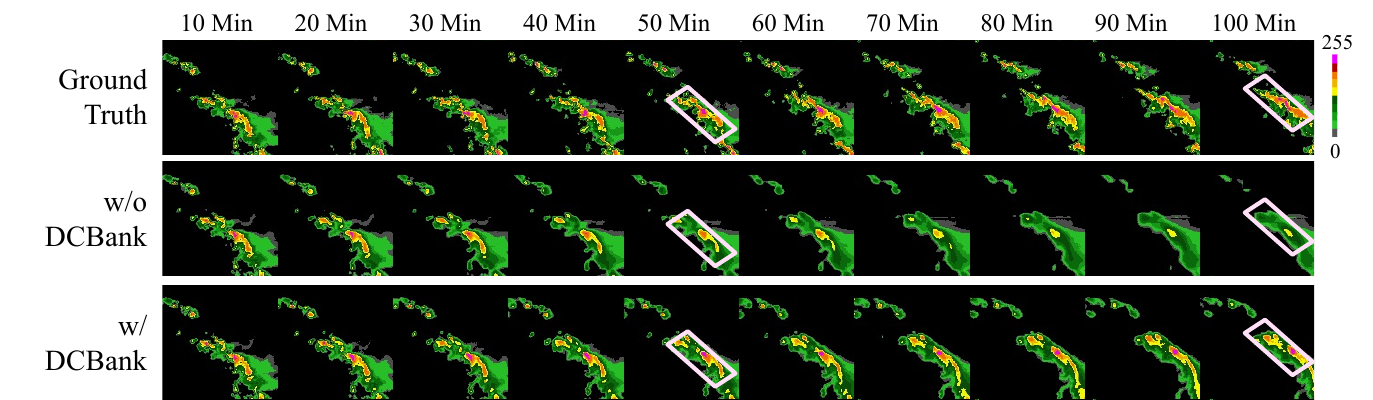}
   \caption{Qualitative comparison on a SEVIR event with and without DCBank. DCBank better preserves fine-grained precipitation structures, especially high-intensity regions.}
   \label{fig:ablation1}
\end{figure}

\subsection{Effectiveness of Content-based Retrieval}
Table~\ref{tab:ablation_retrieval} verifies the effectiveness of content-based retrieval in the Correction-Aware Memory Retrieval module. Removing this component consistently degrades performance across all metrics, with CSI$_{\mathrm{M}}$ dropping from $0.339$ to $0.328$ and high-intensity CSI$_{219}$ decreasing from $0.107$ to $0.090$. This suggests that content similarity is critical for selecting historical states that are relevant to the current precipitation pattern, especially under intense rainfall conditions. By incorporating content-based matching, CAMR retrieves memory states that are both temporally informative and semantically aligned with the current latent representation, thereby providing more effective correction cues for long-horizon precipitation nowcasting.

\begin{table}[!htbp]
\centering
\caption{
Ablation study on SEVIR evaluating the effect of content-based retrieval.
}
\label{tab:ablation_retrieval}
\vspace{1mm}
\small
\setlength{\tabcolsep}{4.5pt}
\renewcommand{\arraystretch}{1.08}
\begin{tabular}{lcccccc}
\toprule
\multirow{2}{*}{Model} 
& \multicolumn{6}{c}{SEVIR} \\
\cmidrule(lr){2-7}
& CSI$_{\mathrm{M}}\uparrow$ 
& CSI$_{181}\uparrow$
& CSI$_{219}\uparrow$
& HSS$\uparrow$
& LPIPS$\downarrow$
& SSIM$\uparrow$ \\
\midrule
w/o content-based retrieval
& 0.328& 0.157 & 0.090 & 0.424 & 0.226 & 0.665 \\

\midrule
\rowcolor{blue!8}
\textbf{McCast} 
& \textbf{0.339} 
& \textbf{0.172} 
& \textbf{0.107} 
& \textbf{0.438} 
& \textbf{0.196} 
& 0.669 \\
\bottomrule
\end{tabular}
\end{table}

\subsection{Effectiveness of DCBank on Different Backbone}
\label{effectiveness_of_dcbank_on_different_backbone}
To further evaluate the generality of DCBank, we integrate it into SimVP~\cite{gao2022simvp}, a representative encoder–hidden–decoder architecture. Specifically, we insert DCBank between the encoder and hidden layers, enabling drift correction in the latent space without modifying the overall prediction pipeline. As reported in Table~\ref{tab:dcbank_on_simvp}, adding DCBank consistently improves the forecasting performance of SimVP, demonstrating that the proposed memory-guided correction mechanism is not tied to a specific backbone. Figure~\ref{fig:simvp_ablation} further provides qualitative evidence. In the highlighted pink frames, SimVP with DCBank better preserves temporally consistent precipitation structures, especially for medium- and high-intensity rainfall, whereas the vanilla SimVP tends to produce weaker and less coherent forecasts.

\begin{table}[!htbp]
\centering
\caption{
Ablation study of DCBank integrated into SimVP on the SEVIR dataset.
}
\label{tab:dcbank_on_simvp}
\vspace{1mm}
\small
\setlength{\tabcolsep}{6.5 pt}
\renewcommand{\arraystretch}{1.08}
\begin{tabular}{lcccccc}
\toprule
\multirow{2}{*}{Model} 
& \multicolumn{6}{c}{SEVIR} \\
\cmidrule(lr){2-7}
& CSI$_{\mathrm{M}}\uparrow$ 
& CSI$_{181}\uparrow$
& CSI$_{219}\uparrow$
& HSS$\uparrow$
& LPIPS$\downarrow$
& SSIM$\uparrow$ \\
\midrule
SimVP  & 0.311  &  0.111  &  0.052  &  0.392  &  0.389  &  0.651  \\

\midrule
\rowcolor{blue!8}
SimVP + \textbf{DCBank}
& \textbf{0.332} 
& \textbf{0.137} 
& \textbf{0.063} 
& \textbf{0.416} 
& \textbf{0.343} 
& \textbf{0.681} \\
\bottomrule
\end{tabular}
\end{table}

\begin{figure}[!htbp]
  \centering
   \includegraphics[width=\linewidth]{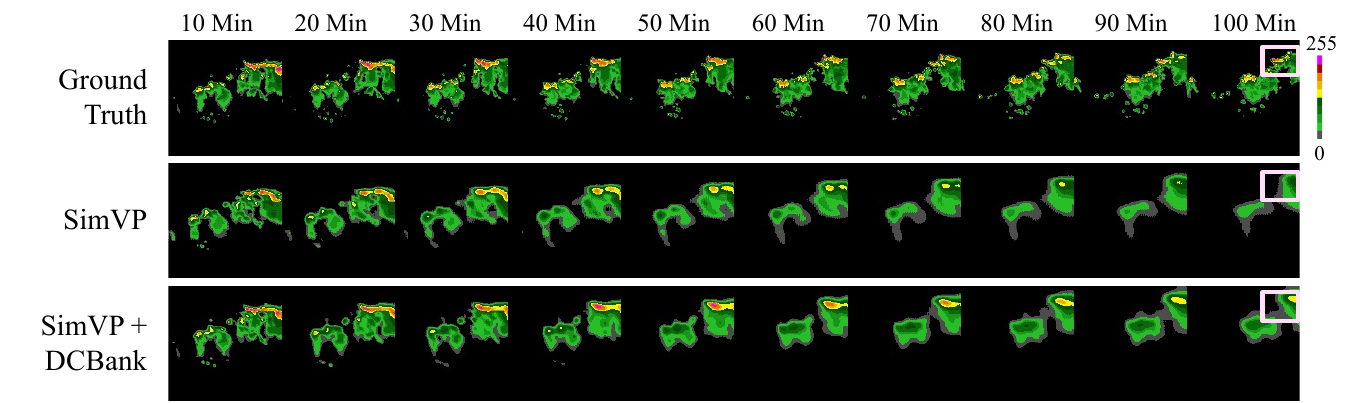}
   \caption{Qualitative comparison of SimVP with and without DCBank on representative SEVIR examples.}
   \label{fig:simvp_ablation}
\end{figure}

\subsection{Hyper-Parameter Sensitivity to $\lambda_\text{drift}$}

We evaluate the sensitivity of McCast to the drift-consistency weight $\lambda_{\text{drift}}$ on the SEVIR dataset. As shown in Table~\ref{tab:hyperparameter_lambda_drift}, McCast achieves the best overall performance when $\lambda_{\text{drift}}=0.3$, obtaining the highest CSI$_{\mathrm{M}}$, CSI$_{181}$, CSI$_{219}$, HSS, and SSIM, as well as the lowest LPIPS. This indicates that a moderate drift-consistency weight provides an effective balance between semantic relevance and temporal drift consistency during memory retrieval. When $\lambda_{\text{drift}}$ is reduced to $0.1$, the performance slightly decreases, suggesting that weak drift-consistency guidance is insufficient to fully exploit temporally coherent historical states. Conversely, increasing $\lambda_{\text{drift}}$ to $0.5$ also leads to a mild degradation, especially in LPIPS, indicating that over-emphasizing drift consistency may weaken semantic matching and introduce less suitable memory retrieval. Overall, the results show that McCast is reasonably stable across different settings, while $\lambda_{\text{drift}}=0.3$ provides the best trade-off and is therefore used as the default configuration.

\begin{table}[!htbp]
\centering
\caption{
Sensitivity analysis of McCast to the drift-consistency weight $\lambda_{\text{drift}}$ on the SEVIR dataset.
}
\label{tab:hyperparameter_lambda_drift}
\vspace{1mm}
\small
\setlength{\tabcolsep}{5pt}
\renewcommand{\arraystretch}{1.08}
\begin{tabular}{lcccccc}
\toprule
\multirow{2}{*}{Model} 
& \multicolumn{6}{c}{SEVIR} \\
\cmidrule(lr){2-7}
& CSI$_{\mathrm{M}}\uparrow$ 
& CSI$_{181}\uparrow$
& CSI$_{219}\uparrow$
& HSS$\uparrow$
& LPIPS$\downarrow$
& SSIM$\uparrow$ \\
\midrule
McCast $\lambda_\text{drift}$ = 0.1 & 0.331& 0.166 & 0.104 & 0.430 & 0.205 & 0.659 \\

McCast $\lambda_\text{drift}$ = 0.3
& \textbf{0.339} 
& \textbf{0.172} 
& \textbf{0.107} 
& \textbf{0.438} 
& \textbf{0.196} 
& \textbf{0.669} \\

McCast $\lambda_\text{drift}$ = 0.5 & 0.335& 0.167 & 0.102 & 0.433 & 0.217 & 0.664 \\

\bottomrule
\end{tabular}
\end{table}

\subsection{Challenging Forecasting Setting}
To further evaluate the robustness of McCast, we conduct experiments under a more challenging long-horizon forecasting setting on SEVIR, where only two historical frames are used to predict the next 20 frames, i.e., $T_i=2$ and $T_o=20$. This setting provides substantially less observational context than the standard protocol, making the model more vulnerable to error accumulation and autoregressive drift.

As shown in Table~\ref{tab:quantitative_2_20}, McCast achieves the best performance on most metrics, including CSI$_{\mathrm{M}}$, CSI$_{181}$, CSI$_{219}$, HSS, and LPIPS. In particular, McCast shows clear improvements at high precipitation thresholds, demonstrating its advantage in preserving intense rainfall structures under limited input context. Although AlphaPre obtains a slightly higher SSIM, McCast achieves substantially better precipitation-oriented scores and perceptual quality, indicating more accurate and reliable forecasts for meteorological nowcasting. These results suggest that the proposed memory-guided drift correction is especially beneficial in challenging long-horizon scenarios where accumulated prediction errors are more severe.

\begin{table}[!htbp]
\centering
\caption{
Quantitative comparison under the challenging $T_i=2$, $T_o=20$ forecasting setting on the SEVIR dataset.
}
\label{tab:quantitative_2_20}
\vspace{1mm}
\small
\setlength{\tabcolsep}{7.5 pt}
\renewcommand{\arraystretch}{1.08}
\begin{tabular}{lcccccc}
\toprule
\multirow{2}{*}{Model} 
& \multicolumn{6}{c}{SEVIR} \\
\cmidrule(lr){2-7}
& CSI$_{\mathrm{M}}\uparrow$ 
& CSI$_{181}\uparrow$
& CSI$_{219}\uparrow$
& HSS$\uparrow$
& LPIPS$\downarrow$
& SSIM$\uparrow$ \\
\midrule
DiffCast  & 0.298  &  0.117  &  0.055  &  0.398  &  0.247  &  0.651 \\

AlphaPre  & \underline{0.323} &  0.125  &  0.057 &  \underline{0.410}  &  0.290  &  \textbf{0.681} \\

PercpCast  & 0.318  &  \underline{0.133}  &  \underline{0.060} &  0.401  &  \underline{0.211}  &  0.668  \\

\midrule
\rowcolor{blue!8}
\textbf{McCast}
& \textbf{0.338} 
& \textbf{0.172} 
& \textbf{0.098} 
& \textbf{0.437} 
& \textbf{0.199} 
& \underline{0.669} \\
\bottomrule
\end{tabular}
\end{table}

\section{Additional Qualitative Results}
\label{appendix:additonal_qualitative_results}

We present additional qualitative results on SEVIR and MeteoNet in Figures~\ref{fig:all_sevir} and~\ref{fig:all_meteo}. Compared with recent baselines, including PercpCast~\cite{fengperceptually} and AlphaPre~\cite{Lin_2025_CVPR}, McCast generates forecasts with clearer precipitation boundaries and better preservation of medium-intensity regions, particularly at longer lead times. When compared with DiffCast~\cite{yu2024diffcast}, both methods capture the overall precipitation structure in the early forecasting stages. However, as the prediction horizon increases, DiffCast tends to suffer from spatial deformation and weakened structural consistency. In contrast, McCast maintains more stable precipitation evolution and preserves fine-scale details over extended horizons.

\begin{figure}[!htbp]
  \centering
   \includegraphics[width=\linewidth]{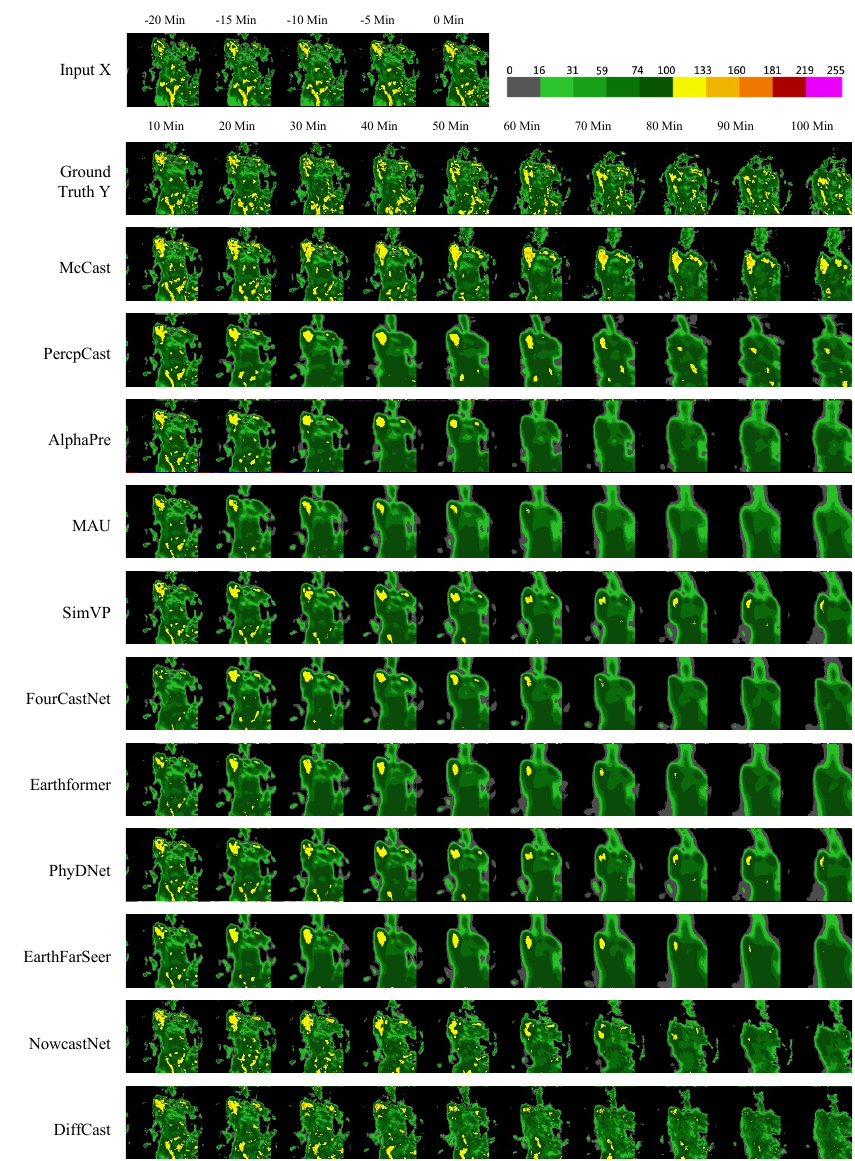}
   \caption{Prediction examples on the SEVIR dataset.}
   \label{fig:all_sevir}
\end{figure}

\begin{figure}[!htbp]
  \centering
   \includegraphics[width=\linewidth]{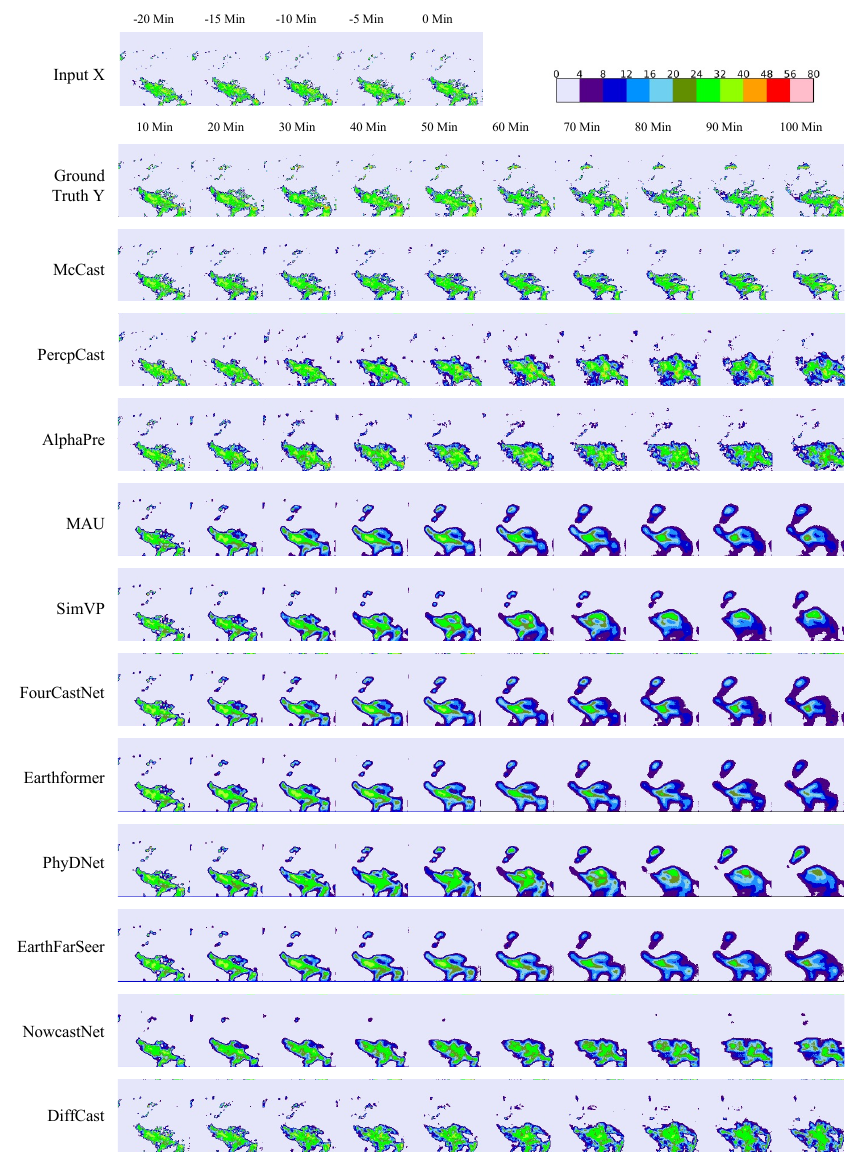}
   \caption{Prediction examples on the MeteoNet dataset.}
   \label{fig:all_meteo}
\end{figure}

%%%%%%%%%%%%%%%%%%%%%%%%%%%%%%%%%%%%%%%%%%%%%%%%%%%%%%%%%%%%

\newpage
\section*{NeurIPS Paper Checklist}

\begin{enumerate}

\item {\bf Claims}
    \item[] Question: Do the main claims made in the abstract and introduction accurately reflect the paper's contributions and scope?
    \item[] Answer: \answerYes{} % Replace by \answerYes{}, \answerNo{}, or \answerNA{}.
    \item[] Justification: The abstract and introduction clearly outline the core contributions of the paper: We propose McCast framework for precipitation nowcasting that explicitly leverages the sequential dependencies among historical states to correct autoregressive drift.

\item {\bf Limitations}
    \item[] Question: Does the paper discuss the limitations of the work performed by the authors?
    \item[] Answer: \answerYes{}  % Replace by \answerYes{}, \answerNo{}, or \answerNA{}.
    \item[] Justification: We include a dedicated discussion of the paper’s limitations in the Appendix~\ref{appendix:limitaiton}

\item {\bf Theory assumptions and proofs}
    \item[] Question: For each theoretical result, does the paper provide the full set of assumptions and a complete (and correct) proof?
    \item[] Answer: \answerNA{} % Replace by \answerYes{}, \answerNo{}, or \answerNA{}.
    \item[] Justification: The paper does not present theoretical results.

    \item {\bf Experimental result reproducibility}
    \item[] Question: Does the paper fully disclose all the information needed to reproduce the main experimental results of the paper to the extent that it affects the main claims and/or conclusions of the paper (regardless of whether the code and data are provided or not)?
    \item[] Answer: \answerYes{} % Replace by \answerYes{}, \answerNo{}, or \answerNA{}.
    \item[] Justification: Our method includes a detailed description of the model architecture in the Section~\ref{section:methodology}, specifies the data sources used in the Section~\ref{section:experiments_discussions}, and provides implementation details in the section~\ref{section:experiments_discussions} and Appendix~\ref{appendix:experimental_details}.

\item {\bf Open access to data and code}
    \item[] Question: Does the paper provide open access to the data and code, with sufficient instructions to faithfully reproduce the main experimental results, as described in supplemental material?
    \item[] Answer: \answerYes{} % Replace by \answerYes{}, \answerNo{}, or \answerNA{}.
    \item[] Justification: We provide the data sources used in the Section~\ref{section:experiments_discussions} and Appendix~\ref{appendix:dataset_details}, and we will release the code upon paper acceptance.

\item {\bf Experimental setting/details}
    \item[] Question: Does the paper specify all the training and test details (e.g., data splits, hyperparameters, how they were chosen, type of optimizer) necessary to understand the results?
    \item[] Answer: \answerYes{} % Replace by \answerYes{}, \answerNo{}, or \answerNA{}.
    \item[] Justification: We provide detailed experimental settings and configurations in the Section~\ref{section:experiments_discussions} and Appendix~\ref{appendix:experimental_details}

\item {\bf Experiment statistical significance}
    \item[] Question: Does the paper report error bars suitably and correctly defined or other appropriate information about the statistical significance of the experiments?
    \item[] Answer: \answerNo{} % Replace by \answerYes{}, \answerNo{}, or \answerNA{}.
    \item[] Justification: Reporting formal error bars is computationally expensive for precipitation nowcasting and is not commonly adopted in this field. Although we do not include statistical error bars, we conducted multiple independent runs and observed consistent convergence behavior with small variations across runs, suggesting that the reported performance is stable and reliable.

\item {\bf Experiments compute resources}
    \item[] Question: For each experiment, does the paper provide sufficient information on the computer resources (type of compute workers, memory, time of execution) needed to reproduce the experiments?
    \item[] Answer: \answerYes{} % Replace by \answerYes{}, \answerNo{}, or \answerNA{}.
    \item[] Justification:  We report the compute resources required for training and evaluation in the Appendix~\ref{appendix:compute_resources} and Appendix~\ref{appendix:model_complexity}.
    
\item {\bf Code of ethics}
    \item[] Question: Does the research conducted in the paper conform, in every respect, with the NeurIPS Code of Ethics \url{https://neurips.cc/public/EthicsGuidelines}?
    \item[] Answer: \answerYes{} % Replace by \answerYes{}, \answerNo{}, or \answerNA{}.
    \item[] Justification: The research conducted in this paper fully complies with the NeurIPS Code of Ethics.

\item {\bf Broader impacts}
    \item[] Question: Does the paper discuss both potential positive societal impacts and negative societal impacts of the work performed?
    \item[] Answer: \answerYes{} % Replace by \answerYes{}, \answerNo{}, or \answerNA{}.
    \item[] Justification: We discuss both the potential positive and negative societal impacts of our work in the Appendix~\ref{appendix:broader_impact}.
    
\item {\bf Safeguards}
    \item[] Question: Does the paper describe safeguards that have been put in place for responsible release of data or models that have a high risk for misuse (e.g., pre-trained language models, image generators, or scraped datasets)?
    \item[] Answer: \answerNA{} % Replace by \answerYes{}, \answerNo{}, or \answerNA{}.
    \item[] Justification: The paper does not involve any models or data with a high risk of misuse.

\item {\bf Licenses for existing assets}
    \item[] Question: Are the creators or original owners of assets (e.g., code, data, models), used in the paper, properly credited and are the license and terms of use explicitly mentioned and properly respected?
    \item[] Answer: \answerYes{} % Replace by \answerYes{}, \answerNo{}, or \answerNA{}.
    \item[] Justification: We show relevant license and attribution information in the Appendix~\ref{appendix:licenses_for_existing_assets}.

\item {\bf New assets}
    \item[] Question: Are new assets introduced in the paper well documented and is the documentation provided alongside the assets?
    \item[] Answer: \answerYes{} % Replace by \answerYes{}, \answerNo{}, or \answerNA{}.
    \item[] Justification:  We will release our newly developed method along with documentation upon acceptance. The asset will include descriptions of model structure, training procedure and evaluation, and will be anonymized at submission time.

\item {\bf Crowdsourcing and research with human subjects}
    \item[] Question: For crowdsourcing experiments and research with human subjects, does the paper include the full text of instructions given to participants and screenshots, if applicable, as well as details about compensation (if any)? 
    \item[] Answer: \answerNA{} % Replace by \answerYes{}, \answerNo{}, or \answerNA{}.
    \item[] Justification: This paper does not involve crowdsourcing and research with human subjects.

\item {\bf Institutional review board (IRB) approvals or equivalent for research with human subjects}
    \item[] Question: Does the paper describe potential risks incurred by study participants, whether such risks were disclosed to the subjects, and whether Institutional Review Board (IRB) approvals (or an equivalent approval/review based on the requirements of your country or institution) were obtained?
    \item[] Answer: \answerNA{} % Replace by \answerYes{}, \answerNo{}, or \answerNA{}.
    \item[] Justification: This paper does not involve crowdsourcing nor research with human subjects.

\item {\bf Declaration of LLM usage}
    \item[] Question: Does the paper describe the usage of LLMs if it is an important, original, or non-standard component of the core methods in this research? Note that if the LLM is used only for writing, editing, or formatting purposes and does \emph{not} impact the core methodology, scientific rigor, or originality of the research, declaration is not required.
    %this research? 
    \item[] Answer: \answerNA{} % Replace by \answerYes{}, \answerNo{}, or \answerNA{}.
    \item[] Justification: The core method development in this research does not involve LLMs as any important, original, or non-standard components. Minor usage was limited to writing assistance only.

\end{enumerate}

\end{document}